\makeatletter
\AddToHook{file/aamas.cls/after}{%
  \def\AAMAS@missingfi{\fi}%
  \ifx\acmConference@editors\AAMAS@missingfi
    \gdef\acmConference@editors{}%
    \expandafter\AAMAS@missingfi
  \fi
  \let\AAMAS@missingfi\@undefined
}
\makeatother

\PassOptionsToPackage{table}{xcolor}
\documentclass[sigconf]{aamas}

\usepackage[T1]{fontenc}
\usepackage[utf8]{inputenc}
\usepackage{microtype}
\usepackage{array}
\usepackage{tabularx}
\usepackage{graphicx}
\usepackage{booktabs}
\usepackage{amsmath}
\usepackage{xcolor}
\usepackage{hyperref}

\usepackage{multirow}
\usepackage{algorithm}
\usepackage{algorithmic}
\usepackage{enumitem}
\usepackage{amsthm}
\usepackage{placeins}
\usepackage{etoolbox}
\usepackage{needspace}
\usepackage{url}
\usepackage{pifont}
\usepackage[table]{xcolor}
\usepackage{xspace}

\newcommand{\xmark}{\ding{55}}

\AtEndPreamble{%
  \hypersetup{
    pdfauthor={},
    colorlinks=true,
    linkcolor=blue,
    citecolor=black,
    urlcolor=blue,
    pdfborder={0 0 0},
    pdfborderstyle={},
    linkbordercolor={1 0 0},
    citebordercolor={0 1 0}
  }%
}

\makeatletter
\AtBeginDocument{%
  \let\HarnessOriginalEqref\eqref
  \renewcommand*{\eqref}[1]{%
    \begingroup
      \hypersetup{linkcolor=black,pdfborder={0 0 1},
        pdfborderstyle={/S/S/W 1}}%
      \HarnessOriginalEqref{#1}%
    \endgroup
  }%
  \pretocmd{\hyper@natlinkstart}{%
    \begingroup
    \hypersetup{pdfborder={0 0 1},pdfborderstyle={/S/S/W 1}}%
  }{}{\PackageError{harness-links}{Cannot style citation links}{}}%
  \apptocmd{\hyper@natlinkend}{\endgroup}{}%
    {\PackageError{harness-links}{Cannot close citation link style}{}}%
}
\makeatother

\numberwithin{equation}{section}

\theoremstyle{acmplain}
\newtheorem{theorem}{Theorem}
\newtheorem{lemma}{Lemma}
\newtheorem{corollary}{Corollary}
\newtheorem{proposition}{Proposition}

\theoremstyle{acmdefinition}
\newtheorem{definition}{Definition}
\newtheorem{assumption}{Assumption}

\newcommand{\Dbel}{D_{\mathrm{belief}}}
\newcommand{\Darr}{D_{\mathrm{arrival}}}
\newcommand{\Dgro}{D_{\mathrm{growth}}}
\newcommand{\BIWM}{\textsc{BIWM}\xspace}
\newcommand{\hibench}{\textsc{HIBench-Code}\xspace}

\makeatletter
\gdef\@copyrightpermission{
  \begin{minipage}{0.2\columnwidth}
   \href{https://creativecommons.org/licenses/by/4.0/}{\includegraphics[width=0.90\textwidth]{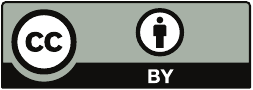}}
  \end{minipage}\hfill
  \begin{minipage}{0.8\columnwidth}
   \href{https://creativecommons.org/licenses/by/4.0/}{This work is licensed under a Creative Commons Attribution International 4.0 License.}
  \end{minipage}
  \vspace{5pt}
}
\makeatother

\setcopyright{ifaamas}
\acmConference[AAMAS 2027]
{Proc.\@ of the 26th International Conference on Autonomous Agents and Multiagent Systems (AAMAS 2027)}
{3--7 May 2027}
{Hanoi, Vietnam}
{M.~Baldoni, F.~Fang, W.~Yeoh, N.~Yorke-Smith (eds.)}

\copyrightyear{2027}
\acmYear{2027}
\acmDOI{}
\acmISBN{}

\acmSubmissionID{}
\submissionType{Research Paper Track}

\title{Measuring Harness-Induced Belief Divergence in Multi-Step LLM Agents}

\author{Haiwen Yi}
\authornote{Equal contribution.}
\affiliation{%
  \institution{University of Toronto}
  \country{Canada}
}

\author{Xinyuan Song}
\authornotemark[1]
\authornote{Corresponding author.}
\affiliation{%
  \institution{Emory University}
  \country{United States}
}

\renewcommand{\shortauthors}{Yi and Song}

\begin{abstract}
Software agents are becoming increasingly important in modern coding workflows, where they solve complex tasks through multi-step interaction with code repositories. A central component of these systems is the execution harness, which governs how the underlying LLM observes, acts on, and verifies its environment. We show that the harness can in turn shape the agent's evolving beliefs about the environment, even when the task, environment, and base LLM remain fixed. To reveal this effect, we introduce a belief-rollout diagnostic that traces how an agent's anticipated world state and future decisions evolve under alternative harnesses. We further separate immediate divergence introduced by the interface from divergence that develops over longer horizons. Across controlled coding tasks include HIBench-Code, SWE-bench Verified and Terminal-Bench, we find substantial cross-harness belief divergence that is often invisible from terminal success alone and can propagate into different downstream actions. Building on this observation, we introduce \textbf{Belief-Invariant World-Modeling (BIWM)}, which exposes evidence hidden or compressed by harness execution and aligns the resulting belief trajectories across different harnesses. Our findings reveal that harness design can systematically shape how software agents understand and act on their environments, providing practical guidance for building and evaluating more reliable coding agents. Our code is available at \url{https://github.com/Hik289/Harness-induce-bias.git}.
\end{abstract}

\ccsdesc[500]{Computing methodologies~Intelligent agents}
\keywords{LLM agents, execution harnesses, belief divergence, agent evaluation, world modeling}

\begin{document}
\maketitle


\section{Introduction}

Large language model agents are becoming increasingly important in modern software-engineering workflows. For example, SWE-agent lets language models navigate repositories, inspect files, edit code, and execute commands through a dedicated agent--computer interface~\citep{yang2024sweagent}. OpenHands similarly provides agents with an integrated environment for interacting with code, terminals, and other software-development tools~\citep{wang2024openhands,wang2025openhands_sdk}. These systems increasingly target realistic repository-level tasks rather than isolated code generation, such as resolving GitHub issues in SWE-bench~\citep{jimenez2023swbench} and completing command-line workflows in Terminal-Bench~\citep{merrill2026terminalbench}. A central component enabling these multi-step interactions is the \emph{execution harness}: the system surrounding the underlying LLM that mediates how it observes, acts on, and verifies its environment. In SWE-agent, for instance, the agent--computer interface determines how repository states and commands are exposed to the model~\citep{yang2024sweagent}; in OpenHands, the runtime similarly structures tool execution and the resulting observations returned to the agent~\citep{wang2025openhands_sdk}. More generally, a harness can determine which observations are exposed, which actions are available, how tool calls are executed, when verification occurs, and how execution feedback is transformed before it reaches the model~\citep{xu2026lifeharness}.

These design choices are not behaviorally neutral. Vats and Golev show that running the same coding model under different agent scaffolds can produce substantial differences in task performance and token consumption~\citep{vats2026scaffoldeffectcodingagents}. Harness Bench further treats the execution harness as an explicit evaluation variable and documents differences across realistic agent workflows~\citep{harnessbench2026}. More recently, Hu et al.\ systematically vary components such as tool interfaces, context management, planning, and auxiliary agent mechanisms, showing that harness configuration can materially affect software-engineering agent performance~\citep{hu2026beyondmodel}. Together, these findings suggest that the harness is not merely implementation infrastructure: it is an experimental variable that can change the behavior observed from an otherwise fixed model.

We argue that these outcome-level effects point to a deeper phenomenon: the harness can shape the agent's evolving \emph{beliefs} about its environment. A multi-step agent uses execution feedback to infer what has happened and decide what to do next~\citep{yao2023react,shinn2023reflexion}. The harness mediates this feedback, so changing the harness can change the agent's interpretation even when the task and base LLM remain fixed. Consider a command that fails during execution. A raw harness may expose the failure directly, whereas a risk-gated harness may prevent the command from executing and instead report that it was blocked. The agent is then given different evidence about the same intended interaction and may form a different belief about the state of the task. Because this belief informs subsequent reasoning, the difference can propagate into later decisions. We refer to this phenomenon as \emph{harness-induced belief divergence}.

Figure~\ref{fig:intuition} illustrates the central intuition: changing the harness changes the evidence available to the agent, which can alter its belief about the task. This effect may remain hidden behind the same terminal outcome, motivating us to examine how beliefs evolve during multi-step reasoning.

\begin{figure}[!ht]
\captionsetup{position=bottom,skip=1.5em}
\centering
\includegraphics[width=\linewidth]{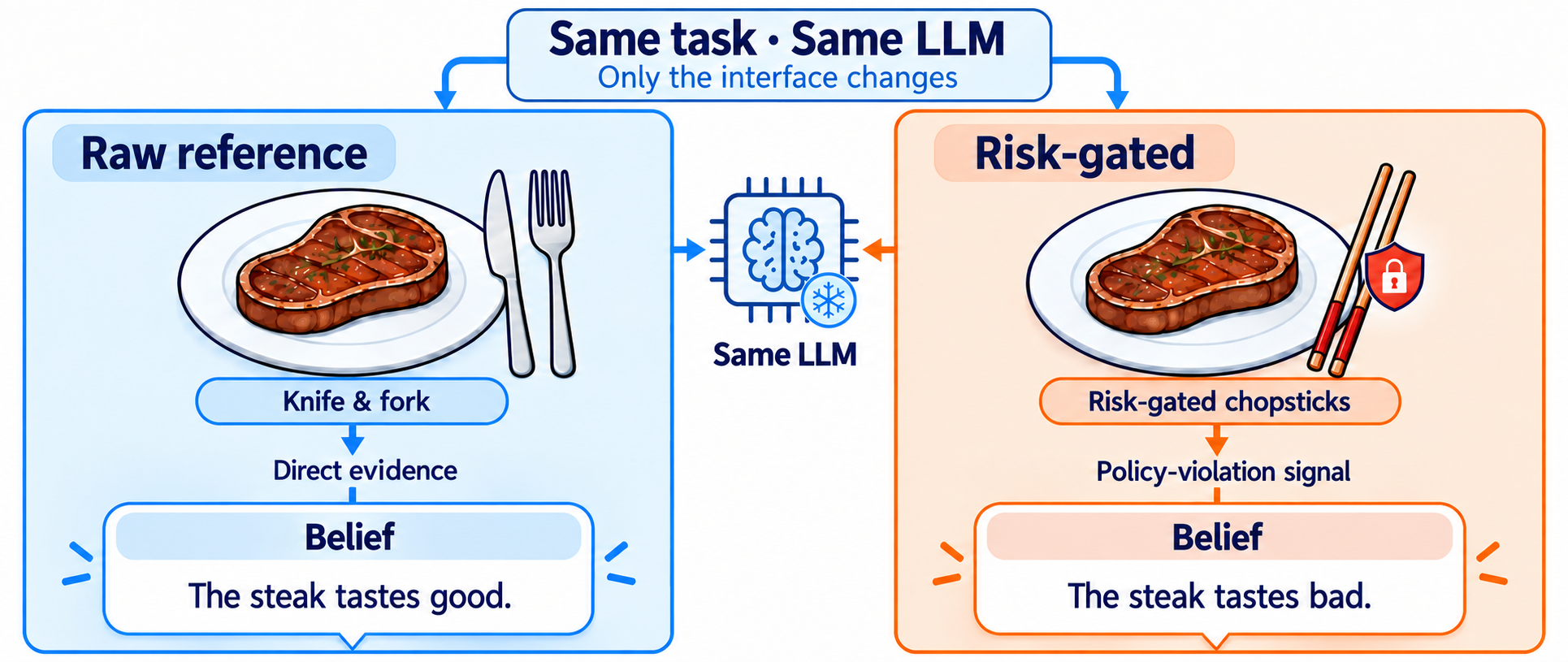}

\caption{\textbf{Harness-induced belief divergence.} The same task and LLM can form different beliefs when the harness changes the interaction evidence.}
\label{fig:intuition}
\vspace{-1em}
\Description{Schematic steak task with two harness interfaces. A knife-and-fork interface and a risk-gated chopstick interface give the same model different evidence about feasibility and expected success for the same task.}
\end{figure}

To measure this effect, we introduce a \emph{belief-rollout diagnostic} that traces how an agent expects a task to evolve under a given harness. Unlike traditional world models that predict environment transitions~\citep{hafner2018planet,hafner2023dreamerv3}, we treat the agent's inference-time world model as an evolving belief trajectory. For each harness, we begin from its visible interaction record and ask the same fixed LLM to recursively forecast $K$ belief updates, following the broader idea of simulated planning with language models~\citep{hao2023rap,yao2023tot}. Each update summarizes the model's current interpretation of the task and its expectation of what will happen next. The predicted belief is then carried forward as context for the following step, allowing harness-induced differences to accumulate over the imagined horizon. We repeat this procedure across harnesses under controlled conditions, isolating how harness-mediated evidence shapes the resulting belief trajectories. We compare the resulting $K$-step belief trajectories across harnesses, where $K$ denotes the number of recursively imagined belief updates rather than live environment interactions.

We study this phenomenon on HIBench-Code and extend the analysis to SWE-bench Verified and Terminal-Bench~\citep{openai2024swebenchverified,merrill2026terminalbench}. Across these settings, alternative harnesses induce substantial belief divergence with distinct horizon-dependent patterns. This divergence also has behavioral consequences: larger belief-growth divergence is associated with greater downstream action divergence. In risk-gated settings, agents may even return to the same class of risky actions after an earlier attempt has been blocked. These results show that harness effects extend beyond terminal outcomes to the belief dynamics that guide subsequent behavior.

Building on these findings, we introduce \textbf{Belief-Invariant World-Modeling (BIWM)}, a no-training protocol for reducing dependence on a single harness view. \BIWM{} first exposes evidence that may be hidden or compressed during execution, making harness-induced differences more observable. This step can increase measured divergence because previously missing evidence becomes explicit. \BIWM{} then aligns beliefs across harness views to improve robustness to these differences. Thus, \BIWM{} separates \emph{evidence exposure} from \emph{belief alignment}, rather than treating divergence reduction as a single operation.

Our main contributions are summarzied as follows:
\begin{itemize}[left = 0em]
    \item We identify and systematically measure \emph{harness-induced belief divergence}, showing that different execution harnesses can induce distinct belief trajectories for the same task and base LLM.

    \item We introduce a \emph{belief-rollout diagnostic} with the $\Darr$/$\Dgro$ decomposition to separate immediate interface effects from divergence that develops over the imagined horizon, and evaluate it across HIBench-Code, SWE-bench Verified, and Terminal-Bench.

    \item We introduce \textbf{Belief-Invariant World-Modeling (BIWM)}, a no-training protocol that exposes harness-mediated evidence and aligns beliefs across harness views, providing a practical approach for diagnosing and reducing harness dependence.
\end{itemize}

\begin{table*}[!ht]
\centering
\small
\setlength{\tabcolsep}{3.5pt}
\renewcommand{\arraystretch}{1.12}
\caption{Comparison with closely related work. \checkmark{} indicates that the work explicitly studies or provides the corresponding capability, while \xmark{} indicates that it does not. ACI denotes agent--computer interface. ``Belief reports'' refers to structured intermediate reports elicited over imagined continuations rather than only executed trajectories.}

\label{tab:related-work-comparison}
\resizebox{\textwidth}{!}{%
\begin{tabular}{@{}lcccccccc@{}}
\toprule
\textbf{Work} &
\textbf{\shortstack{Harness / ACI\\Variation}} &
\textbf{\shortstack{Executed-Trace\\Analysis}} &
\textbf{\shortstack{Belief\\Reports}} &
\textbf{\shortstack{Sequential\\Uncertainty}} &
\textbf{\shortstack{Arrival /\\Growth}} &
\textbf{\shortstack{Censored\\Evidence}} &
\textbf{\shortstack{Matched Cross-Harness\\Contrast}} &
\textbf{\shortstack{Alignment /\\Robustness Audit}} \\
\midrule
\multicolumn{9}{@{}l}{\textit{Harness design and agent evaluation}} \\
SWE-agent~\citep{yang2024sweagent} & \checkmark & \checkmark & \xmark & \xmark & \xmark & \xmark & \xmark & \xmark \\
Life-Harness~\citep{xu2026lifeharness} & \checkmark & \checkmark & \xmark & \xmark & \xmark & \xmark & \checkmark & \xmark \\
HAL~\citep{kapoor2025hal} & \checkmark & \checkmark & \xmark & \xmark & \xmark & \xmark & \checkmark & \xmark \\
Harness-Bench~\citep{yao2026harnessbench} & \checkmark & \checkmark & \xmark & \xmark & \xmark & \xmark & \checkmark & \xmark \\
AgentBoard~\citep{ma2024agentboard} & \xmark & \checkmark & \xmark & \xmark & \xmark & \xmark & \xmark & \xmark \\
\midrule
\multicolumn{9}{@{}l}{\textit{World modeling and sequential uncertainty}} \\
WebDreamer / SiRA~\citep{gu2024webdreamer,deng2025simplanningworldmodel} & \xmark & \checkmark & \checkmark & \xmark & \xmark & \xmark & \xmark & \xmark \\
VAGEN~\citep{shen2025vagen} & \xmark & \checkmark & \checkmark & \xmark & \xmark & \xmark & \xmark & \xmark \\
UProp / SAUP~\citep{duan2025uprop,zhao2025saup} & \xmark & \xmark & \xmark & \checkmark & \xmark & \xmark & \xmark & \xmark \\
\midrule
\multicolumn{9}{@{}l}{\textit{Runtime safety and verification}} \\
ToolSafe / Verifier Tax~\citep{toolsafe2026,sah2026verifiertax} & \checkmark & \checkmark & \xmark & \xmark & \xmark & \checkmark & \xmark & \xmark \\
\midrule
\rowcolor{gray!15}
\textbf{This paper} & \checkmark & \checkmark & \checkmark & \checkmark & \checkmark & \checkmark & \checkmark & \checkmark \\
\bottomrule
\end{tabular}%
}

\end{table*}

\section{Related Work}

\paragraph{Agent harnesses and interfaces.}
Agent systems intertwine model reasoning with external execution interfaces. ReAct established the now-standard thought--act--observe loop~\citep{yao2023react}, while Reflexion demonstrated that interaction feedback can reshape subsequent agent behaviour~\citep{shinn2023reflexion}. In software engineering, SWE-agent elevated the agent--computer interface to a first-class design variable and showed that interface choices materially affect coding performance~\citep{yang2024sweagent}. Complementary work has explored searching over modular agent design spaces~\citep{wang2024agentsquare}, adapting runtime harness parameters~\citep{xu2026lifeharness}, and building production-oriented agent platforms~\citep{wang2024openhands,wang2025openhands_sdk}. Program-repair agents further show that automatic retries and repair feedback are central to software-agent behaviour~\citep{bouzenia2024repairagent}. These studies establish that the harness is consequential rather than incidental. Our focus is different: rather than optimizing a harness for task success, we hold the task and LLM fixed and measure how harness variation changes the intermediate world-model belief trajectory.

\paragraph{LLM world models and multi-step reasoning.}
Framing an LLM as an implicit world model has a growing literature. RAP connects language-model reasoning with world-model planning~\citep{hao2023rap}, while PlaNet and Dreamer demonstrate how learned latent world models can support long-horizon planning~\citep{hafner2018planet,hafner2023dreamerv3}. Recent work probes whether LLMs develop stable internal state representations~\citep{vafa2024llmworldrep}, trains multi-turn world-model reasoning~\citep{shen2025vagen}, and studies LLM-agent planning through world-model-style simulation~\citep{huang2024planningsurvey,deng2025simplanningworldmodel}. WebDreamer~\citep{gu2024webdreamer} applies this perspective to web navigation, using LLM-simulated futures to reduce real-environment interaction. Our setting differs in that we neither train a separate parametric world model nor treat a single verifier output as a world model. Instead, we operationalize the world model through repeated inference-time belief updates whose structured intermediate states can be audited and compared across matched harnesses. This enables us to study whether changing the execution harness alone changes the resulting belief trajectory.

\paragraph{Agent benchmarks and evaluation.}
SWE-bench~\citep{jimenez2023swbench}, WebArena~\citep{zhou2023webarena}, OSWorld~\citep{xie2024osworld}, and $\tau$-bench~\citep{yao2024taubench} provide widely used evaluations for coding, web navigation, desktop control, and tool-use agents. API-Bank~\citep{li2023apibank}, ToolLLM~\citep{qin2023toolllm}, Gorilla/APIBench~\citep{patil2024gorilla}, and the Berkeley Function Calling Leaderboard~\citep{patil2025bfcl} further evaluate API selection, tool invocation, and function-call correctness, while RedCode~\citep{he2024redcode} makes risky code execution an explicit evaluation dimension. These benchmarks primarily report execution-level quantities such as task success, pass rate, API-call correctness, and cost. Process-oriented evaluation provides a complementary perspective: AgentBoard~\citep{ma2024agentboard}, for example, introduces fine-grained progress measurements for multi-turn agents, while broader surveys cover evaluation across coding, tool-use, and embodied settings~\citep{luo2025llmagentsurvey,li2024surveymultiagent}. Audits of agent-benchmark reporting further indicate that protocol and infrastructure details can remain under-specified, complicating cross-system comparisons~\citep{benchmarkaudit2026}. Our evaluation complements these approaches by holding the task and model fixed, systematically varying the harness, and measuring the resulting structured belief trajectory rather than only the terminal execution outcome.

\paragraph{Calibration, safety, and guardrails.}
Uncertainty calibration in LLMs builds on classical probability calibration and language-model self-evaluation~\citep{guo2017calibration,kadavath2022lmknow}, with more recent work studying uncertainty and its propagation in multi-step settings~\citep{uqsurvey2025,duan2025uprop,wang2025steca,hiremath2026cdur}. In parallel, safety-oriented agent systems study when an agent should terminate~\citep{checkyourself2025}, how to apply step-level guardrails~\citep{toolsafe2026}, how to intercept risky tool calls~\citep{agenttrust2026}, how to emulate risky execution in a sandbox~\citep{ruan2023toolemu}, and how to evaluate prompt-injection attacks and defenses in tool-using agents~\citep{debenedetti2024agentdojo}. Other work studies failures of safety mechanisms under long-context or agentic risk states~\citep{hadeliya2025refusalsfail,wilhelm2026riskagentstates}. These interventions do more than constrain execution: blocking, verification, repair, and selective observation also change the evidence available to subsequent reasoning. Our analysis treats this evidence mediation as measurable rather than incidental. \BIWM{} makes blocked-action, repair, verification, and shadow-execution evidence observable, and then separately applies cross-harness alignment to reduce dependence on any single harness view.

\paragraph{Comparison with prior work.}
Table~\ref{tab:related-work-comparison} positions our work relative to prior studies of harness design, world modeling, sequential uncertainty, and runtime safety. Unlike prior work, we jointly measure harness-conditioned belief trajectories, decompose divergence into arrival and growth components, and use \BIWM{} to distinguish evidence exposure from cross-harness alignment.

\section{Preliminaries}
\label{sec:preliminaries}

\paragraph{Harness-conditioned comparison.}
We study how the execution harness shapes an agent's beliefs while holding the underlying task and model fixed. We fix the task $\mathcal{T}$, environment $\mathcal{E}$, base LLM $M$, belief elicitation procedure, initial interaction record, and matched random seed, and vary only the harness
\begin{equation}
H=(O_H,A_H,G_H,R_H,L_H),
\label{eq:harness}
\end{equation}
where $O_H$, $A_H$, $G_H$, $R_H$, and $L_H$ denote the observation, action, gating, repair, and logging components, respectively. Together, they determine the interaction and execution evidence exposed to the model. Appendix~\ref{app:harness} gives the full harness specification and canonicalization procedure.

\paragraph{Harness-conditioned belief rollout.}
Given a harness-visible interaction record, the fixed model produces structured beliefs over task state, failure mode, constraints, numeric forecasts, and the recommended next action. Starting from the same matched initial record, we recursively elicit the horizon-$K$ belief rollout
\begin{equation}
B_{0:K}^{H}=(B_0^H,\ldots,B_K^H).
\label{eq:belief-rollout}
\end{equation}
This rollout captures the model's anticipated belief trajectory rather than $K$ additional live environment interactions. Harness-mediated events such as blocking, repair, and verification remain represented in the forecast context, so $B_K^H$ is a model forecast rather than environment ground truth. Each horizon is elicited independently from the same matched initial record and seed. Appendix~\ref{app:belief-rollout} gives the full belief space and rollout process.

\paragraph{Belief divergence.}
For matched harnesses $H_a$ and $H_b$, let $B_K^{H_a}$ and $B_K^{H_b}$ be their final-step beliefs from the rollouts in eq.~\eqref{eq:belief-rollout}. We compare them using bounded distances over categorical state, failure mode, constraint content, numeric forecasts, and recommended action, collected as
\begin{equation}
\mathbf{D}(H_a,H_b;K)=(D_{\mathrm{cat}},D_{\mathrm{fail}},D_{\mathrm{set}},D_{\mathrm{num}},D_{\mathrm{act}})^\top.
\label{eq:component-divergence}
\end{equation}
With a non-negative weight vector $\mathbf{w}$ satisfying $\mathbf{1}^\top\mathbf{w}=1$, the overall belief divergence is
\begin{equation}
\Dbel(H_a,H_b;K)=\mathbf{w}^\top\mathbf{D}(H_a,H_b;K).
\label{eq:main-dbelief}
\end{equation}
This provides an overall measure of harness-induced belief disagreement while keeping its component sources separately inspectable. Appendix~\ref{app:beliefdivergence} gives the component definitions, normalization, and recomputation details.

\paragraph{Arrival and growth divergence.}
We separate harness-induced disagreement into \emph{arrival} differences in constraints and recommended actions and \emph{growth} differences in state, failure, and numeric forecasts. For the component vector $\mathbf{D}$ in Eq.~\eqref{eq:component-divergence} and weight vector $\mathbf{w}$ in Eq.~\eqref{eq:main-dbelief}, let $\mathbf{s}_{\mathrm{arr}}$ and $\mathbf{s}_{\mathrm{gro}}$ select the corresponding components, and define
\begin{equation}
\alpha_{\mathrm{arr}}=\frac{1}{\mathbf{w}^{\top}\mathbf{s}_{\mathrm{arr}}}, \qquad \alpha_{\mathrm{gro}}=\frac{1}{\mathbf{w}^{\top}\mathbf{s}_{\mathrm{gro}}}.
\label{eq:arrival-growth-normalizers}
\end{equation}
The normalized readouts are
\begin{equation}
\Darr(H_a,H_b;K)=\alpha_{\mathrm{arr}}(\mathbf{w}\odot\mathbf{s}_{\mathrm{arr}})^{\top}\mathbf{D}(H_a,H_b;K),
\label{eq:main-darr}
\end{equation}
and
\begin{equation}
\Dgro(H_a,H_b;K)=\alpha_{\mathrm{gro}}(\mathbf{w}\odot\mathbf{s}_{\mathrm{gro}})^{\top}\mathbf{D}(H_a,H_b;K),
\label{eq:main-dgro}
\end{equation}
where $\odot$ denotes element-wise multiplication. Since $\mathbf{s}_{\mathrm{arr}}+\mathbf{s}_{\mathrm{gro}}=\mathbf{1}$, the divergence in Eq.~\eqref{eq:main-dbelief} decomposes exactly as
\begin{equation}
\Dbel(H_a,H_b;K)=\alpha_{\mathrm{arr}}^{-1}\Darr(H_a,H_b;K)+\alpha_{\mathrm{gro}}^{-1}\Dgro(H_a,H_b;K).
\label{eq:main-decomposition}
\end{equation}
Thus, $\Darr$ captures disagreement in constraints and recommended actions, while $\Dgro$ captures disagreement in belief quantities that may evolve across the forecast horizon. Appendix~\ref{app:arrival-growth} provides the detailed definitions of the selection vectors, normalization, decomposition, bounds, and associated theoretical properties.

\paragraph{Scope.}
The matched setup in Eq.~\eqref{eq:harness} attributes differences between the belief rollouts in Eq.~\eqref{eq:belief-rollout} to the harness by holding the task, environment, model, elicitation procedure, initial record, and randomness fixed. An individual harness component is isolated only when the compared harnesses differ solely in that component. The \BIWM{} intervention in Section~\ref{sec:method} is analyzed separately because it additionally varies instrumentation and cross-view evidence. Appendix~\ref{app:identification} provides the detailed matched-comparison assumptions, component-level isolation conditions, and harness-level attribution result, while Appendix~\ref{app:censoring} formalizes censoring under harness gating and its induced belief divergence.

\section{Method}
\label{sec:method}

\begin{figure*}[!ht]
  \centering
  \includegraphics[width=0.9\linewidth]{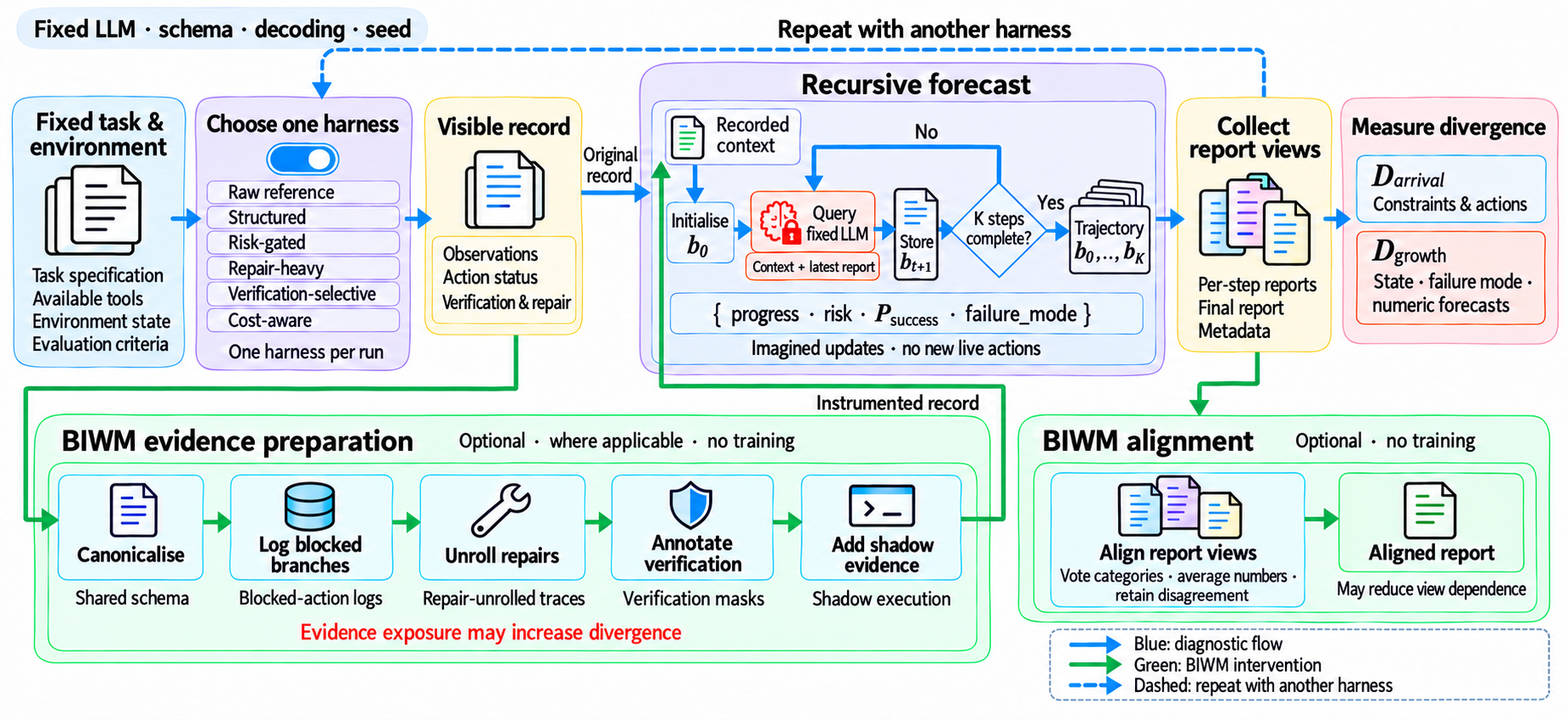}
    \vspace{-1em}
  \caption{\textbf{Harness-to-belief instrumentation pipeline.} \textbf{(a)} A fixed task, LLM, schema, decoding procedure, and seed isolate harness-induced arrival and growth divergence through recursive belief forecasts. \textbf{(b)} \BIWM{} evidence preparation exposes blocked, repaired, verified, and shadow evidence, which may increase measured divergence. \textbf{(c)} \BIWM{} alignment reconciles cross-harness report views and may reduce view dependence.}
  \label{fig:pipeline}

  \Description{Pipeline from six execution harnesses through recorded evidence to a fixed language model and its recursive structured belief forecast. \BIWM{} components one through five expose missing or compressed evidence; component six aligns the resulting harness views.}
\end{figure*}

\paragraph{Harness-specific evidence interventions.}
Figure~\ref{fig:pipeline}(a) evaluates six instances of the harness in Eq.~\eqref{eq:harness}: raw reference, structured, risk-gated, repair-heavy, verification-selective, and cost-aware. They differ through observation presentation, action gating, repair exposure, verification evidence, and cost-driven evidence removal while satisfying Assumption~\ref{ass:matched-comparison}. For a component-level intervention, Definition~\ref{def:component-isolation} requires
\begin{equation}
H_a^{(j)}\neq H_b^{(j)},\qquad H_a^{(i)}=H_b^{(i)}\ \text{for all}\ i\neq j.
\label{eq:method-component-control}
\end{equation}
Thus, each controlled pair changes one harness mechanism while retaining the matched rollout protocol.

\paragraph{Recursive belief rollout.}
For each task, harness and $K$, we apply the rollout in Definition~\ref{def:rollout} and Eq.~\eqref{eq:belief-rollout}. The next belief is conditioned on the current belief and harness-visible record as specified in Eq.~\eqref{eq:belief-transition}; recommended actions follow Eq.~\eqref{eq:belief-action}. In the gated case $G_H(A_t^H)=\varnothing$, the forecast retains the blocked-action evidence according to Eq.~\eqref{eq:blocked-rollout} rather than inserting an unrealized environment outcome. Algorithm~\ref{alg:rollout} implements the process shown in Fig.~\ref{fig:pipeline}(a).

\begin{algorithm}[!ht]
\caption{$K$-step harness-conditioned belief rollout}
\label{alg:rollout}
\begin{algorithmic}[1]
\STATE Initialise the rollout from the matched initial record.
\FOR{$t=0$ \textbf{to} $K-1$}
  \STATE Construct the harness-visible record using Eq.~\eqref{eq:harness}.
  \STATE Elicit the next belief according to Eq.~\eqref{eq:belief-transition}.
  \STATE Read the recommended action using Eq.~\eqref{eq:belief-action}.
  \STATE Apply gating, repair, verification, and logging according to the active harness.
  \STATE Record visible evidence and the resulting structured belief.
  \STATE Extend the imagined record without executing a new live benchmark action.
\ENDFOR
\STATE \textbf{Return} the rollout in Eq.~\eqref{eq:belief-rollout}.
\end{algorithmic}
\end{algorithm}

Each horizon is elicited independently from the same matched initial record and seed. Final beliefs are compared using Eq.~\eqref{eq:main-dbelief}, with the arrival and growth readouts in Eqs.~\eqref{eq:main-darr} and~\eqref{eq:main-dgro}.

\paragraph{\BIWM{} evidence exposure.}
Figure~\ref{fig:pipeline}(b) applies canonicalization, blocked-action logging, repair unrolling, verification annotation, and shadow execution to recover omitted or compressed evidence. These operations preserve censored branches in Eq.~\eqref{eq:blocked-rollout} and may increase the divergence in Eq.~\eqref{eq:main-dbelief} by exposing previously hidden disagreement.

\paragraph{\BIWM{} cross-harness alignment.}
Figure~\ref{fig:pipeline}(c) aligns matched report views by categorical voting and numeric averaging while retaining disagreement as uncertainty. Alignment may reduce view dependence, whereas evidence exposure may increase measured divergence; thus, \BIWM{} is not a uniform divergence-reduction operator. Details are provided in Appendix~\ref{app:method-details}.

\section{Experiments}
\label{sec:experiments}

\subsection{Metrics}
\label{app:metrics}

Following prior work that evaluates multi-step agents beyond terminal success through process-level measurements~\citep{ma2024agentboard,duan2025uprop,benchmarkaudit2026}, we measure harness-conditioned belief differences using the five components in Eq.~\eqref{eq:component-divergence}. $D_{\mathrm{cat}}$ measures disagreement in task progress, risk state, and recoverability as defined in Eq.~\eqref{eq:dcat}; $D_{\mathrm{fail}}$ measures failure-mode disagreement using Eq.~\eqref{eq:dfail}; $D_{\mathrm{set}}$ compares known, satisfied, and violated constraints using Eq.~\eqref{eq:dset}; $D_{\mathrm{num}}$ measures disagreement in numeric forecasts using Eq.~\eqref{eq:dnum}; and $D_{\mathrm{act}}$ measures disagreement in the recommended next action using Eq.~\eqref{eq:dact}.

The overall divergence $\Dbel$ is computed by Eq.~\eqref{eq:main-dbelief}. We additionally report the arrival and growth readouts in Eqs.~\eqref{eq:main-darr} and~\eqref{eq:main-dgro}: $\Darr$ summarizes constraint and action disagreement, while $\Dgro$ summarizes categorical-state, failure-mode, and numeric-forecast disagreement. Reporting the component distances together with these aggregate readouts allows changes in $\Dbel$ to be traced to specific parts of the belief state rather than interpreted as a single undifferentiated score.

\subsection{Harness Mechanism Study}
\label{sec:phase1}

We evaluate \hibench-v0, a controlled coding benchmark with eight tasks spanning off-by-one errors, missing null checks, import cycles, and destructive-action traps, complementing outcome-oriented benchmarks such as SWE-bench~\citep{jimenez2023swbench}, AgentBench~\citep{liu2023agentbench}, and AgentBoard~\citep{ma2024agentboard}. The evaluation crosses six harnesses, four rollout horizons, and three random seeds while holding the matched quantities in Assumption~\ref{ass:matched-comparison} fixed.

The raw reference exposes unfiltered observations and executed actions, following fully visible execution used in coding-agent evaluation~\citep{jimenez2023swbench,yang2024sweagent}. Structured adds parsed tracebacks, targeted verification summaries, and explicit intermediate state. Risk-gated intercepts high-risk actions before execution and records policy-violation metadata~\citep{he2024redcode,toolsafe2026,agenttrust2026}. Repair-heavy records failure--repair--recovery transitions~\citep{bouzenia2024repairagent}. Verification-selective records whether verification is applied at each state, while cost-aware reduces expensive verification and long-running checks under budget constraints~\citep{sah2026verifiertax}. \BIWM{} is protocol instrumentation rather than a separate harness: it augments these views with canonicalization, blocked-action logging, repair-unrolled traces, verification masks, shadow execution, and cross-harness alignment.

Each mediated harness is compared with the raw reference using Eq.~\eqref{eq:main-dbelief}, with arrival and growth reported by Eqs.~\eqref{eq:main-darr} and~\eqref{eq:main-dgro}. Table~\ref{tab:phase1} reports final-step divergence averaged over task--seed pairs.

\begin{table}[!ht]
\centering
\small
\caption{\textbf{Controlled harness mechanism study.} Final-step divergence relative to the raw reference harness across rollout horizons. Gray cells mark the largest $\Dbel$ or $\Dgro$ at each horizon.}
\label{tab:phase1}
\resizebox{\columnwidth}{!}{%
\begin{tabular}{llrrrr}
\toprule
Comparison & Metric & $K=1$ & $K=3$ & $K=5$ & $K=8$ \\
\midrule
Raw--structured
 & $\Dbel$ & 0.409 & 0.437 & \cellcolor{gray!20}\textbf{0.479} & \cellcolor{gray!20}\textbf{0.478} \\
 & $\Darr$ & 0.999 & 0.990 & 0.998 & 1.000 \\
 & $\Dgro$ & 0.156 & 0.200 & \cellcolor{gray!20}\textbf{0.257} & \cellcolor{gray!20}\textbf{0.254} \\
\midrule
Raw--risk-gated
 & $\Dbel$ & 0.398 & 0.422 & 0.408 & 0.411 \\
 & $\Darr$ & 0.995 & 0.990 & 0.999 & 1.000 \\
 & $\Dgro$ & 0.142 & 0.179 & 0.154 & 0.158 \\
\midrule
Raw--repair-heavy
 & $\Dbel$ & \cellcolor{gray!20}\textbf{0.436} & 0.425 & 0.402 & 0.387 \\
 & $\Darr$ & 0.997 & 1.000 & 0.997 & 0.975 \\
 & $\Dgro$ & \cellcolor{gray!20}\textbf{0.195} & 0.178 & 0.147 & 0.134 \\
\midrule
Raw--verification-selective
 & $\Dbel$ & 0.431 & \cellcolor{gray!20}\textbf{0.444} & 0.404 & 0.462 \\
 & $\Darr$ & 0.994 & 0.995 & 0.996 & 1.000 \\
 & $\Dgro$ & 0.190 & \cellcolor{gray!20}\textbf{0.208} & 0.151 & 0.231 \\
\midrule
Raw--cost-aware
 & $\Dbel$ & 0.421 & 0.405 & 0.414 & 0.398 \\
 & $\Darr$ & 0.997 & 0.998 & 0.997 & 0.999 \\
 & $\Dgro$ & 0.174 & 0.152 & 0.164 & 0.140 \\
\bottomrule
\end{tabular}%
}

\end{table}

Table~\ref{tab:phase1} shows why the decomposition in Eq.~\eqref{eq:main-decomposition} is necessary. $\Darr$ is already near its upper bound at $K=1$, producing a large arrival floor in $\Dbel$, whereas $\Dgro$ separates the subsequent belief dynamics: structured execution increases from $0.156$ to $0.257$ by $K=5$, verification-selective execution reaches $0.231$ at $K=8$, and repair-heavy execution decreases from $0.195$ to $0.134$. Thus, harness effects are not merely immediate differences in constraints or recommended actions; the growth component in Eq.~\eqref{eq:main-dgro} reveals distinct horizon-dependent belief trajectories that are partially obscured by the overall divergence.

\subsection{Generalization Across Public Benchmarks}
\label{sec:benchmark_extended}

We test whether the harness effects in Section~\ref{sec:phase1} persist on SWE-bench Verified~\citep{jimenez2023swbench,openai2024swebenchverified} and Terminal-Bench~\citep{merrill2026terminalbench}. Tasks are grouped by execution regime, and all six harnesses are evaluated at $K\in\{3,5\}$ under Assumption~\ref{ass:matched-comparison}. Table~\ref{tab:public_grouped} reports mean divergence using Eq.~\eqref{eq:main-dbelief}.

\begin{table}[!ht]
\centering
\small
\caption{\textbf{Harness-induced belief divergence across public benchmarks.}
Mean $\Dbel$ by task group. Gray cells mark the largest divergence within each group.}

\label{tab:public_grouped}
\resizebox{\columnwidth}{!}{%
\begin{tabular}{llccccc}
\toprule
Benchmark & Group & Structured & Risk-gated & Repair-heavy & Verify-selective & Cost-aware \\
\midrule
\multirow{3}{*}{SWE-bench Verified}
& A\_bug
& 0.613 & 0.640 & \cellcolor{gray!20}\textbf{0.669} & 0.656 & 0.655 \\
& B\_feature
& 0.613 & 0.641 & 0.646 & 0.645 & \cellcolor{gray!20}\textbf{0.656} \\
& C\_refactor
& 0.617 & 0.616 & 0.662 & \cellcolor{gray!20}\textbf{0.671} & 0.648 \\
\midrule
\multirow{3}{*}{Terminal-Bench}
& X\_risky
& 0.557 & 0.591 & \cellcolor{gray!20}\textbf{0.640} & 0.587 & 0.601 \\
& Y\_timeout
& 0.514 & 0.514 & 0.578 & 0.564 & \cellcolor{gray!20}\textbf{0.594} \\
& Z\_verif
& 0.516 & 0.495 & \cellcolor{gray!20}\textbf{0.559} & 0.531 & 0.532 \\
\bottomrule
\end{tabular}%
}

\end{table}

Table~\ref{tab:public_grouped} shows that harness-induced divergence persists on both public benchmarks, while the largest effect varies across task regimes. We next examine whether this transfer is also visible in individual belief components.

\begin{table}[!ht]
\centering
\scriptsize
\setlength{\tabcolsep}{3pt}
\caption{\textbf{Horizon-dependent divergence on SWE-bench Verified.}
Gray cells highlight the largest failure-mode divergence.}

\label{tab:swbench}
\resizebox{0.8\columnwidth}{!}{%
\begin{tabular}{lcrrrrr}
\toprule
Comparison & $K$ & $\Dbel$ & $D_{\mathrm{cat}}$ &
$D_{\mathrm{fail}}$ & $D_{\mathrm{set}}$ & $D_{\mathrm{num}}$ \\
\midrule
Raw--structured
& 3 & 0.325 & 0.150 & 0.300 & 0.797 & 0.041 \\
Raw--structured
& 5 & 0.373 & 0.150 & 0.600 & 0.943 & 0.051 \\
Raw--risk-gated
& 3 & 0.330 & 0.125 & 0.400 & 0.844 & 0.050 \\
Raw--risk-gated
& 5 & 0.383 & 0.092 &
\cellcolor{gray!20}\textbf{0.800} & 0.927 & 0.049 \\
\bottomrule
\end{tabular}%
}

\end{table}

On SWE-bench Verified~\citep{jimenez2023swbench,openai2024swebenchverified}, $D_{\mathrm{fail}}$ increases from $0.300$ to $0.600$ for structured execution and from $0.400$ to $0.800$ for risk gating as $K$ increases from $3$ to $5$, showing that additional rollout steps can amplify failure-attribution differences even when changes in Eq.~\eqref{eq:main-dbelief} are smaller.

\begin{table}[!ht]
\centering
\small
\caption{\textbf{Cross-benchmark divergence at $K=5$.}
Overall, failure-mode, and constraint-set divergence for two harness interventions. Gray cells mark the largest component differences.}

\label{tab:crossbench3}
\resizebox{\columnwidth}{!}{%
\begin{tabular}{llrrr}
\toprule
Pair & Metric & \hibench{} & Terminal-Bench & SWE-bench Verified \\
\midrule
\multirow{3}{*}{Raw--structured}
& $\Dbel$
& 0.479 & 0.381 & 0.373 \\
& $D_{\mathrm{fail}}$
& \cellcolor{gray!20}\textbf{0.708} & 0.500 & 0.600 \\
& $D_{\mathrm{set}}$
& \cellcolor{gray!20}\textbf{0.998} & 0.676 & 0.943 \\
\midrule
\multirow{3}{*}{Raw--risk-gated}
& $\Dbel$
& 0.408 & 0.363 & 0.383 \\
& $D_{\mathrm{fail}}$
& 0.458 & 0.500 & \cellcolor{gray!20}\textbf{0.800} \\
& $D_{\mathrm{set}}$
& \cellcolor{gray!20}\textbf{0.999} & 0.676 & 0.927 \\
\bottomrule
\end{tabular}%
}

\end{table}

Table~\ref{tab:crossbench3} further shows that the composition of Eq.~\eqref{eq:main-dbelief} is benchmark-dependent: \hibench{} produces the largest $D_{\mathrm{set}}$, whereas SWE-bench Verified~\citep{jimenez2023swbench,openai2024swebenchverified} produces the largest risk-gated $D_{\mathrm{fail}}$ at $0.800$. Together, Tables~\ref{tab:public_grouped}--\ref{tab:crossbench3} support a mechanism-specific generalization claim: harness-induced belief divergence persists across public benchmarks, but its dominant harness and belief component depend on the execution regime.

\subsection{\BIWM{} Effects}
\label{sec:keyresults}

We evaluate \BIWM{} using the growth readout in Eq.~\eqref{eq:main-dgro}. Table~\ref{tab:biwm_split} separates evidence exposure in Fig.~\ref{fig:pipeline}(b) from cross-harness alignment in Fig.~\ref{fig:pipeline}(c) at $K=5$.

\begin{table}[!ht]

\centering
\small
\caption{\textbf{\BIWM{} exposure and alignment effects at $K=5$.}
Positive $\Delta_{\mathrm{exp}}$ exposes hidden divergence; negative $\Delta_{\mathrm{rob}}$ indicates reduced harness dependence.}

\label{tab:biwm_split}
\resizebox{\columnwidth}{!}{%
\begin{tabular}{lccccc}
\toprule
Harness
& Baseline $\Dgro$
& Exposure $\Dgro$
& $\Delta_{\mathrm{exp}}$
& Aligned $\Dgro$
& $\Delta_{\mathrm{rob}}$ \\
\midrule
Structured
& 0.312 & 0.318 & $+$0.006
& 0.298 & $-$0.014 \\
Risk-gated
& 0.487 & 0.492 & $+$0.005
& 0.451 & \cellcolor{gray!20}\textbf{$-$0.036} \\
Repair-heavy
& 0.147 & 0.373
& \cellcolor{gray!20}\textbf{$+$0.226}
& 0.112 & $-$0.035 \\
Verify-selective
& 0.151 & 0.367 & $+$0.216
& 0.134 & $-$0.017 \\
Cost-aware
& 0.103 & 0.251 & $+$0.148
& 0.089 & $-$0.014 \\
\bottomrule
\end{tabular}%
}

\end{table}

Evidence exposure is largest for repair-heavy ($+0.226$) and verification-selective ($+0.216$), whereas alignment yields the largest reduction for risk-gated execution ($-0.036$). Thus, \BIWM{} exposes hidden disagreement before reducing harness-specific dependence.

\paragraph{Component ablation.}
Table~\ref{tab:biwm} localizes the exposure effect to individual instrumentation components and the full components 1--5 stack.

\begin{table}[!ht]
\centering
\small
\caption{\textbf{\BIWM{} component ablation at $K=5$.}
Gray cells mark the largest $\Dgro$ within each harness block.}

\label{tab:biwm}
\resizebox{\columnwidth}{!}{%
\begin{tabular}{llrrr}
\toprule
Base view & Instrumentation & $\Dbel$ & $\Darr$ & $\Dgro$ \\
\midrule
Structured
& None
& 0.479 & 0.998 & 0.257 \\
Structured
& Canonicalisation
& 0.490 & 0.994 & 0.274 \\
Structured
& \BIWM{}-full
& 0.536 & 0.983 & \cellcolor{gray!20}\textbf{0.344} \\
\midrule
Risk-gated
& None
& 0.408 & 0.999 & 0.154 \\
Risk-gated
& Blocked-action log
& 0.427 & 0.999 & 0.182 \\
Risk-gated
& \BIWM{}-full
& 0.537 & 0.988 & \cellcolor{gray!20}\textbf{0.344} \\
\midrule
Repair-heavy
& None
& 0.402 & 0.997 & 0.147 \\
Repair-heavy
& Repair-unrolled log
& 0.533 & 0.978 & 0.342 \\
Repair-heavy
& \BIWM{}-full
& 0.553 & 0.974 & \cellcolor{gray!20}\textbf{0.373} \\
\midrule
Verification-selective
& None
& 0.404 & 0.996 & 0.151 \\
Verification-selective
& Verification mask
& 0.434 & 0.986 & 0.198 \\
Verification-selective
& \BIWM{}-full
& 0.553 & 0.986 & \cellcolor{gray!20}\textbf{0.367} \\
\midrule
Cost-aware
& None
& 0.414 & 0.997 & 0.164 \\
Cost-aware
& Shadow execution
& 0.433 & 0.999 & 0.191 \\
Cost-aware
& \BIWM{}-full
& 0.539 & 0.987 & \cellcolor{gray!20}\textbf{0.347} \\
\bottomrule
\end{tabular}%
}

\end{table}

All mechanism-specific interventions increase $\Dgro$ with comparatively small changes in $\Darr$; repair-unrolled logging produces the largest single-component effect.

\paragraph{Transfer across benchmarks.}
We next test whether the exposure effects transfer from \hibench{} to Terminal-Bench~\citep{merrill2026terminalbench}. Table~\ref{tab:biwmcross} compares each uninstrumented harness with its \BIWM{}-full counterpart at $K=5$.

\begin{table}[!ht]
\centering
\small
\caption{\textbf{\BIWM{}-full effects across benchmarks at $K=5$.}
Positive $\Delta\Dbel$ denotes additional divergence exposed by instrumentation relative to the corresponding uninstrumented view.}
\label{tab:biwmcross}

\resizebox{\columnwidth}{!}{%
\begin{tabular}{lrrrrrr}
\toprule
Base view
& \multicolumn{3}{c}{\hibench{} ($n=21$ paired)}
& \multicolumn{3}{c}{Terminal-Bench ($n=10$)} \\
\cmidrule(lr){2-4}\cmidrule(lr){5-7}
& Base & \BIWM{}-full & $\Delta\Dbel$
& Base & \BIWM{}-full & $\Delta\Dbel$ \\
\midrule
Structured
& 0.479 & 0.525 & +0.046
& 0.381 & 0.447 & +0.066 \\
Risk-gated
& 0.408 & 0.536 & +0.128
& 0.363 & 0.430 & +0.067 \\
Repair-heavy
& 0.402 & 0.537 & +0.135
& 0.421 & 0.450 & +0.029 \\
Verify-selective
& 0.404 & 0.551 & \cellcolor{gray!20}\textbf{+0.147}
& 0.393 & 0.439 & +0.046 \\
Cost-aware
& 0.414 & 0.527 & +0.112
& 0.360 & 0.463 & \cellcolor{gray!20}\textbf{+0.103} \\
\bottomrule
\end{tabular}%
}

\end{table}

The full stack increases $\Dbel$ for every harness on both benchmarks, but the effect does not transfer uniformly across mechanisms. Repair-heavy shows the strongest attenuation, from $+0.135$ on \hibench{} to $+0.029$ on Terminal-Bench, suggesting that the additional repair evidence exposed by \BIWM{} is more task-dependent. In contrast, cost-aware remains nearly unchanged ($+0.112$ versus $+0.103$), indicating that evidence suppressed by cost-aware execution persists across both task families. Thus, the cross-benchmark result supports transfer of the overall exposure effect, but not of its mechanism-specific magnitude.

\paragraph{Single-component transfer.}
Table~\ref{tab:biwmsingle} tests whether the individual instrumentation effects in Table~\ref{tab:biwm} persist across benchmarks.

\begin{table}[!ht]
\centering
\small
\caption{\textbf{Single-component \BIWM{} transfer.}
$\Delta\Dbel$ is the mean change from the uninstrumented view; positive/total counts cases with $\Delta\Dbel>0$.}

\label{tab:biwmsingle}
\resizebox{\columnwidth}{!}{%
\begin{tabular}{lrrrr}
\toprule
Component
& \multicolumn{2}{c}{\hibench{}}
& \multicolumn{2}{c}{Terminal-Bench} \\
\cmidrule(lr){2-3}\cmidrule(lr){4-5}
& $\Delta\Dbel$ & Positive/total
& $\Delta\Dbel$ & Positive/total \\
\midrule
Canonicalisation
& +0.011 & 15/24
& +0.005 & 4/10 \\
Blocked-action log
& +0.019 & 13/24
& \cellcolor{gray!20}\textbf{+0.034} & 6/10 \\
Repair-unrolled log
& \cellcolor{gray!20}\textbf{+0.131} & 22/24
& +0.007 & 5/10 \\
Verification mask
& +0.030 & 18/24
& +0.025 & \cellcolor{gray!20}\textbf{8/10} \\
Shadow execution
& +0.019 & 10/24
& -0.002 & 4/10 \\
\bottomrule
\end{tabular}%
}

\end{table}

Component effects transfer unevenly across task families. Repair-unrolled logging is strongest on \hibench{} ($+0.131$) but falls to $+0.007$ on Terminal-Bench, whereas verification masks remain positive on both benchmarks and affect $18/24$ and $8/10$ cases, respectively.

\paragraph{Alignment across horizons.}
Figure~\ref{fig:alignment} extends the $K=5$ alignment result in Table~\ref{tab:biwm_split} across rollout depths.

\begin{figure}[!ht]
\centering
\includegraphics[width=0.9\linewidth]{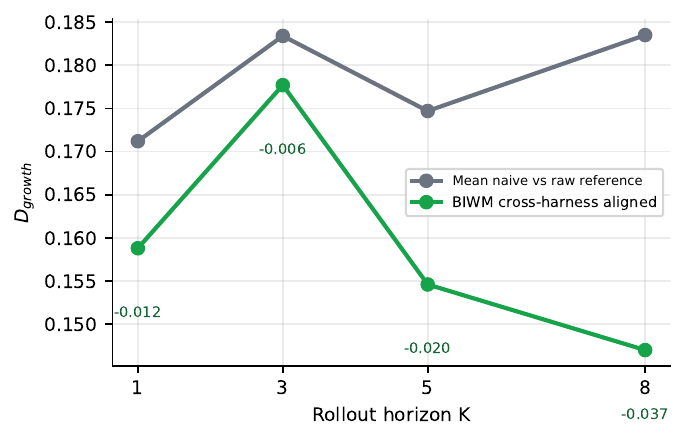}

\caption{\textbf{Cross-harness alignment reduces growth divergence across horizons.}
The reduction increases from $0.012$ at $K=1$ to $0.037$ at $K=8$.}
\label{fig:alignment}
\vspace{-1em}
\Description{Mean individual-view and aligned-view growth divergence across rollout horizons.}
\end{figure}

\paragraph{Behavioral consequences.}
We test whether larger growth divergence in Eq.~\eqref{eq:main-dgro} is associated with greater disagreement in the recommended next action. Recommendations are mapped to six coarse action categories using the same classifier across harnesses. Table~\ref{tab:behavior_main} compares matched pairs in the lowest and highest $\Dgro$ quartiles.

\begin{table}[!ht]
\centering
\small
\caption{\textbf{Growth divergence is associated with action disagreement.}
Matched pairs with larger $\Dgro$ exhibit substantially higher action divergence.}

\label{tab:behavior_main}
\resizebox{0.82\columnwidth}{!}{%
\begin{tabular}{lccc}
\toprule
$\Dgro$ group & Pairs & Action divergence & Change \\
\midrule
Lowest quartile
& 211 & 0.280 & -- \\
Highest quartile
& 210
& \cellcolor{gray!20}\textbf{0.595}
& \cellcolor{gray!20}\textbf{$+0.315$} \\
\bottomrule
\end{tabular}%
}

\end{table}

Across 840 matched step-level pairs, action divergence increases from $0.280$ to $0.595$ between the lowest and highest $\Dgro$ quartiles. Across harnesses, the rate ranges from $0.321$ for repair-heavy to $0.601$ for structured execution. On risk-gated Terminal-Bench~\citep{merrill2026terminalbench}, $42/60=0.700$ blocked high-risk steps are followed by a same-class risky proposal within three steps, showing that blocking execution does not necessarily eliminate the corresponding action tendency.

\paragraph{Temporal self-consistency.}
Finally, we test whether \BIWM{} changes the consistency between an early world-model forecast and the same model's later imagined belief state. For failure prediction, we compare the step-0 failure-attractor probability with whether the final belief at $K=5$ assigns a non-null failure mode, using AUROC. This measures within-rollout consistency rather than calibration against environment outcomes~\citep{guo2017calibration,kadavath2022lmknow}.

\begin{table}[!ht]
\centering
\small
\caption{\textbf{Temporal self-consistency at $K=5$.}
AUROC compares the initial failure-attractor forecast with the final imagined failure state; it does not measure environment-grounded accuracy.}

\label{tab:g4auroc}
\resizebox{0.78\columnwidth}{!}{%
\begin{tabular}{lccc}
\toprule
Harness & Base & \BIWM{}-full & $\Delta$ \\
\midrule
Structured
& 0.519 & 0.591 & +0.073 \\
Verify-selective
& 0.450 & 0.556 & +0.106 \\
Cost-aware
& 0.174 & 0.428
& \cellcolor{gray!20}\textbf{+0.254} \\
\bottomrule
\end{tabular}%
}

\end{table}

\BIWM{}-full increases temporal agreement for all three non-degenerate harnesses, with the largest change under cost-aware execution. However, the cost-aware AUROC remains below $0.5$, and the final belief is itself model-generated rather than environment ground truth. The result therefore provides evidence of changed internal consistency, not improved predictive accuracy or calibration.

\subsection{Long-Horizon Diagnostics}
\label{sec:longhorizon}

We extend the controlled \hibench{} study to
$K\in\{1,3,5,8,12,16,20\}$ using the same task and harness families,
base LLM, elicitation template, decoding procedure, and logging protocol.
The overlapping horizons provide a separate replication of the short-horizon
pattern in Table~\ref{tab:phase1}.

\begin{table}[!ht]
\centering
\small
\caption{\textbf{Long-horizon belief divergence relative to the raw harness.}
$\Dbel$ does not grow monotonically with rollout horizon.}

\label{tab:longhorizon}
\resizebox{\columnwidth}{!}{%
\begin{tabular}{lrrrrrrr}
\toprule
Comparison & $K=1$ & $K=3$ & $K=5$ & $K=8$ & $K=12$ & $K=16$ & $K=20$ \\
\midrule
Raw--structured
& 0.404 & 0.445 & 0.457 & \cellcolor{gray!20}\textbf{0.494}
& 0.482 & 0.485 & 0.479 \\

Raw--risk-gated
& 0.368 & 0.453 & 0.365 & 0.454
& 0.430 & 0.430 & \cellcolor{gray!20}\textbf{0.484} \\

Raw--repair-heavy
& \cellcolor{gray!20}\textbf{0.436} & 0.445 & 0.379 & 0.394
& 0.387 & 0.400 & 0.413 \\

Raw--verification-selective
& 0.420 & \cellcolor{gray!20}\textbf{0.474} & 0.413
& 0.474 & 0.409 & 0.427 & 0.426 \\

Raw--cost-aware
& 0.425 & 0.397 & 0.381 & 0.430
& 0.388 & 0.422 & 0.431 \\
\bottomrule
\end{tabular}%
}

\end{table}

Table~\ref{tab:longhorizon} shows no systematic long-horizon growth in
Eq.~\eqref{eq:main-dbelief}: each harness exhibits local increases and
decreases rather than monotone accumulation. We therefore inspect the
component distances from Eq.~\eqref{eq:component-divergence}.

\begin{table}[!ht]
\centering
\small
\caption{\textbf{Component-level long-horizon diagnostics.}
Action and constraint distances remain near saturation, while categorical,
failure-mode, and numeric distances continue to vary.}

\label{tab:componentsK20}
\resizebox{\columnwidth}{!}{%
\begin{tabular}{lrrrrrrr}
\toprule
Component & $K=1$ & $K=3$ & $K=5$ & $K=8$ & $K=12$ & $K=16$ & $K=20$ \\
\midrule
$\Dbel$
& 0.410 & 0.443 & 0.399 & 0.449 & 0.419 & 0.433 & 0.447 \\
$D_{\mathrm{cat}}$
& 0.087 & 0.135 & 0.087 & 0.131 & 0.079 & 0.079 & 0.110 \\
$D_{\mathrm{fail}}$
& 0.500 & 0.575 & \cellcolor{gray!20}\textbf{0.325}
& 0.575 & 0.475 & 0.500 & 0.575 \\
$D_{\mathrm{set}}$
& 0.993 & 0.983 & 0.996 & 0.995 & 0.998 & 1.000 & 0.983 \\
$D_{\mathrm{num}}$
& 0.044 & 0.081 & 0.099 & 0.100 & 0.099
& \cellcolor{gray!20}\textbf{0.136} & 0.126 \\
$D_{\mathrm{act}}$
& 1.000 & 1.000 & 1.000 & 1.000 & 1.000 & 1.000 & 1.000 \\
\bottomrule
\end{tabular}%
}

\end{table}

Table~\ref{tab:componentsK20} shows that scalar stability does not imply
component convergence. $D_{\mathrm{act}}$ remains saturated and
$D_{\mathrm{set}}$ remains near one, while $D_{\mathrm{num}}$ increases
from $0.044$ at $K=1$ to $0.136$ at $K=16$ and remains $0.126$ at $K=20$.
Meanwhile, $D_{\mathrm{fail}}$ transiently falls to $0.325$ at $K=5$
before returning to $0.575$ at $K=8$ and $K=20$. Thus,
Eq.~\eqref{eq:main-dbelief} can appear stable even when individual belief
components remain horizon-sensitive.

\subsection{Metric Sensitivity Analysis}
\label{sec:sensitivity}

We test whether the main $\Dbel$ pattern depends on the component weights in Definition~\ref{def:dbel}. Setting A uses the implemented weight vector $\mathbf{w}_{A}=(0.30,0.15,0.25,0.25,0.05)^\top$. Setting B downweights the arrival-heavy set component and reallocates weight toward categorical and numeric divergence, while setting C uses uniform weights $\mathbf{w}_{C}=(0.20,0.20,0.20,0.20,0.20)^\top$.

\begin{table}[!ht]
\centering
\small
\caption{\textbf{Metric-weight sensitivity at $K=5$.}
Changing component weights alters the scale and middle ordering of $\Dbel$, while structured remains the largest-divergence harness.}

\label{tab:sensitivity}
\resizebox{\columnwidth}{!}{%
\begin{tabular}{lccc}
\toprule
Harness
& A (implemented)
& B (arrival-down)
& C (uniform) \\
\midrule
Structured
& \cellcolor{gray!20}\textbf{0.479}
& \cellcolor{gray!20}\textbf{0.492}
& \cellcolor{gray!20}\textbf{0.597} \\
Risk-gated
& 0.408 & 0.404 & 0.523 \\
Repair-heavy
& 0.402 & 0.426 & 0.540 \\
Verify-selective
& 0.404 & 0.440 & 0.549 \\
Cost-aware
& 0.414 & 0.396 & 0.514 \\
\bottomrule
\end{tabular}%
}

\end{table}

Structured remains the largest-divergence harness under all three settings, while the middle ordering changes with $\mathbf{w}$. Thus, the absolute scale and fine-grained ranking of $\Dbel$ are weight-sensitive, motivating the accompanying $\Darr$, $\Dgro$, and component-level analyses.

\subsection{Cross-Harness Agreement Structure}
\label{sec:correlation}

The preceding experiments compare each mediated harness with the raw reference. We complement this analysis by measuring agreement between all harness-conditioned beliefs. For each harness pair, we average categorical agreement over task progress, risk state, and recoverability.

\begin{table}[!ht]
\centering
\small
\caption{\textbf{Pairwise categorical agreement across harnesses.}
Entries report mean agreement over task progress, risk state, and recoverability at $K=1$ and $K=8$.}

\label{tab:catcorr}
\resizebox{0.9\columnwidth}{!}{%
\begin{tabular}{lcccccc}
\toprule
$K=1$ & Raw & Struct. & Risk & Repair & Verify & Cost \\
\midrule
Raw
& 1.000 & 0.694 & 0.694 & 0.514 & 0.500 & 0.542 \\
Struct.
& 0.694 & 1.000 & \cellcolor{gray!20}\textbf{0.736}
& 0.639 & 0.639 & 0.694 \\
Risk
& 0.694 & 0.736 & 1.000 & 0.597 & 0.639 & 0.667 \\
Repair
& 0.514 & 0.639 & 0.597 & 1.000 & 0.653 & 0.681 \\
Verify
& 0.500 & 0.639 & 0.639 & 0.653 & 1.000
& \cellcolor{gray!20}\textbf{0.736} \\
Cost
& 0.542 & 0.694 & 0.667 & 0.681 & 0.736 & 1.000 \\
\midrule
$K=8$ & Raw & Struct. & Risk & Repair & Verify & Cost \\
\midrule
Raw
& 1.000 & 0.500 & 0.639
& \cellcolor{gray!20}\textbf{0.778}
& 0.528 & 0.736 \\
Struct.
& 0.500 & 1.000 & 0.556 & 0.514
& \cellcolor{gray!20}\textbf{0.403} & 0.569 \\
Risk
& 0.639 & 0.556 & 1.000 & 0.639 & 0.542
& \cellcolor{gray!20}\textbf{0.750} \\
Repair
& 0.778 & 0.514 & 0.639 & 1.000 & 0.653 & 0.736 \\
Verify
& 0.528 & 0.403 & 0.542 & 0.653 & 1.000 & 0.528 \\
Cost
& 0.736 & 0.569 & 0.750 & 0.736 & 0.528 & 1.000 \\
\bottomrule
\end{tabular}%
}

\end{table}

The pairwise structure changes with rollout depth. At $K=1$, structured and risk-gated agree at $0.736$, while verification-selective and cost-aware also agree at $0.736$. By $K=8$, structured becomes more distinct, reaching only $0.403$ agreement with verification-selective, whereas risk-gated and cost-aware retain $0.750$ agreement and repair-heavy and cost-aware retain $0.736$.

These results show that harness-conditioned beliefs are neither identical nor independent. Their structured agreement and disagreement provide complementary views for the cross-harness alignment evaluated in Fig.~\ref{fig:alignment}.

\section{Conclusion}

Execution harnesses are a measurable source of belief divergence in multi-step LLM agents, even when the task and model are fixed. Our arrival--growth decomposition reveals differences obscured by terminal outcomes, and larger growth divergence is associated with greater downstream action disagreement. \BIWM{} further separates evidence exposure from cross-harness alignment, showing that instrumentation can first reveal hidden divergence and then reduce harness-specific dependence. These results motivate evaluating agents through both outcomes and harness-conditioned belief trajectories.

\bibliography{references}
\bibliographystyle{ACM-Reference-Format}
\clearpage
\appendix


\section{Full Formalization}
\label{app:formalization}

\subsection{Harness Specification}
\label{app:harness}

Let $H$ be the set of finite histories of observation--action--outcome triples, and let $A$ and $O$ denote the full action and observation spaces of the environment.

\begin{definition}[Harness]
\label{def:harness}
The harness in Eq.~\eqref{eq:harness} is
\begin{equation}
H=\bigl(O_H,A_H,G_H,R_H,L_H\bigr),
\end{equation}
where:
\begin{itemize}[leftmargin=2em,itemsep=2pt]
  \item $O_H:H\to O_H$ is the observation map, where $O_H\subseteq O$;
  \item $A_H\subseteq A$ is the action component;
  \item $G_H:A\to A\cup\{\varnothing\}$ is the gating component, where $G_H(a)=\varnothing$ indicates that action $a$ is blocked before execution;
  \item $R_H:H\to H$ is the repair component, which rewrites the history presented to the model after failure or correction;
  \item $L_H:H\to L$ is the logging component, which projects the interaction history onto the auditable record space $L$.
\end{itemize}
\end{definition}

The raw reference harness $H_{\mathrm{raw}}$ sets
\begin{equation}
O_{H_{\mathrm{raw}}}=\mathrm{id},\qquad A_{H_{\mathrm{raw}}}=A,\qquad G_{H_{\mathrm{raw}}}=\mathrm{id},\qquad R_{H_{\mathrm{raw}}}=\mathrm{id},
\end{equation}
and $L_{H_{\mathrm{raw}}}$ records the full history verbatim.

\begin{assumption}[Canonical Embedding]
\label{ass:canonical}
There exist canonical schema spaces $O^*$ and $A^*$ and deterministic maps
\begin{equation}
\iota_O:\bigcup_H O_H\to O^*,\qquad \iota_A:\bigcup_H A_H\to A^*,
\end{equation}
such that semantically equivalent observation and action representations are mapped to the same canonical representation.
\end{assumption}

The maps $\iota_O$ and $\iota_A$ remove only schema- and syntax-level variation. Differences introduced by observation filtering, action availability, gating, repair, or logging remain unchanged. In particular, if $G_H(a)=\varnothing$, canonicalization preserves the blocked-action event rather than replacing it with an unrealized environment outcome.

\subsection{\texorpdfstring{$K$-Step Harness-Conditioned Belief Rollout}{K-Step Harness-Conditioned Belief Rollout}}
\label{app:belief-rollout}

We formalize the belief rollout in Eq.~\eqref{eq:belief-rollout} using the five belief components introduced in Eq.~\eqref{eq:component-divergence}.

\begin{definition}[Belief Space]
\label{def:beliefspace}
A belief state is
\begin{equation}
B_t^H=\bigl(B_{t,\mathrm{cat}}^H,B_{t,\mathrm{fail}}^H,B_{t,\mathrm{set}}^H,B_{t,\mathrm{num}}^H,B_{t,\mathrm{act}}^H\bigr)\in B,
\label{eq:belief-state}
\end{equation}
where the five components represent categorical task state, failure mode, constraint content, numeric forecasts, and the recommended next action, respectively. These are exactly the components compared by $\mathbf{D}$ in Eq.~\eqref{eq:component-divergence}; their field-level definitions and ranges are given in Appendix~\ref{app:metrics}.
\end{definition}

\begin{definition}[$K$-Step Harness-Conditioned Belief Rollout]
\label{def:rollout}
Fix the matched task $T$, base LLM $M$, initial record $h_0$, random seed, and elicitation procedure, and vary only the harness $H$ in Eq.~\eqref{eq:harness}. The rollout $B_{0:K}^H$ in Eq.~\eqref{eq:belief-rollout} is defined recursively by the conditional distributions
\begin{equation}
P(B_0^H\mid T,M,h_0,H),
\label{eq:belief-initial}
\end{equation}
and
\begin{equation}
P(B_{t+1}^H\mid B_t^H,h_t,T,M,H),\qquad t=0,\ldots,K-1,
\label{eq:belief-transition}
\end{equation}
where $h_t$ is the harness-visible forecast record at step $t$. The recommended action is the action component of the current belief,
\begin{equation}
A_t^H=B_{t,\mathrm{act}}^H.
\label{eq:belief-action}
\end{equation}
\end{definition}

The next forecast record is conditioned on the harness components in Eq.~\eqref{eq:harness}: $O_H$ determines the visible observation, $A_H$ the available actions, $G_H$ whether the recommended action is admitted or blocked, $R_H$ the repaired record, and $L_H$ the retained execution evidence.

The rollout represents a belief forecast rather than $K$ additional live environment interactions. When an action is blocked, the subsequent belief is conditioned on the blocked-action evidence rather than an unrealized environment outcome:
\begin{equation}
P(B_{t+1}^H\mid B_t^H,h_t,G_H(A_t^H)=\varnothing,T,M,H).
\label{eq:blocked-rollout}
\end{equation}
Each horizon $K$ is elicited independently from the same matched initial record and seed.

Conditioned on the current belief and harness-visible record, the rollout satisfies
\begin{equation}
P(B_{t+1}^H\mid B_{0:t}^H,h_{0:t},T,M,H)=P(B_{t+1}^H\mid B_t^H,h_t,T,M,H).
\label{eq:first-order-rollout}
\end{equation}
Thus, the augmented process $(B_t^H,h_t)$ is first-order, while the belief sequence alone need not be Markov because $B_t^H$ need not contain the complete record $h_t$. For matched harnesses $H_a$ and $H_b$, the final beliefs $B_K^{H_a}$ and $B_K^{H_b}$ are compared by Eq.~\eqref{eq:main-dbelief}.

\subsection{Belief Divergence}
\label{app:beliefdivergence}

We now instantiate the five component distances collected in $\mathbf{D}(H_a,H_b;K)$ in Eq.~\eqref{eq:component-divergence}. The comparison is between the final beliefs $B_K^{H_a}$ and $B_K^{H_b}$ produced by the matched $K$-step rollouts in Definition~\ref{def:rollout}. Their categorical, failure, set-valued, numeric, and action components are those defined in Definition~\ref{def:beliefspace}.

\begin{definition}[Component Distances]
\label{def:components}
Let $B_K^{H_a}$ and $B_K^{H_b}$ be the matched final beliefs from Definition~\ref{def:rollout}.

For their categorical components $B_{K,\mathrm{cat}}^{H_a}$ and $B_{K,\mathrm{cat}}^{H_b}$, let $(p_a,r_a,c_a)$ and $(p_b,r_b,c_b)$ denote the corresponding progress, risk, and recoverability fields. Their normalized distance is
\begin{equation}
D_{\mathrm{cat}}(H_a,H_b;K)=\frac{1}{3}\left(\frac{|p_a-p_b|}{4}+\frac{|r_a-r_b|}{2}+\frac{|c_a-c_b|}{2}\right).
\label{eq:dcat}
\end{equation}

For the failure-mode components $B_{K,\mathrm{fail}}^{H_a}=f_a$ and $B_{K,\mathrm{fail}}^{H_b}=f_b$, define
\begin{equation}
D_{\mathrm{fail}}(H_a,H_b;K)=\mathbf{1}\{f_a\neq f_b\}.
\label{eq:dfail}
\end{equation}

For the set-valued components $B_{K,\mathrm{set}}^{H_a}=(\mathcal{K}_a,\mathcal{S}_a,\mathcal{V}_a)$ and $B_{K,\mathrm{set}}^{H_b}=(\mathcal{K}_b,\mathcal{S}_b,\mathcal{V}_b)$, define
\begin{equation}
I_{ab}=|\mathcal{K}_a\cap\mathcal{K}_b|+|\mathcal{S}_a\cap\mathcal{S}_b|+|\mathcal{V}_a\cap\mathcal{V}_b|.
\label{eq:set-intersection}
\end{equation}
\begin{equation}
U_{ab}=|\mathcal{K}_a\cup\mathcal{K}_b|+|\mathcal{S}_a\cup\mathcal{S}_b|+|\mathcal{V}_a\cup\mathcal{V}_b|.
\label{eq:set-union}
\end{equation}
Their constraint-set distance is
\begin{equation}
D_{\mathrm{set}}(H_a,H_b;K)=1-\frac{I_{ab}}{U_{ab}},
\label{eq:dset}
\end{equation}
with $D_{\mathrm{set}}(H_a,H_b;K)=0$ when $U_{ab}=0$.

For the numeric components $B_{K,\mathrm{num}}^{H_a}=\mathbf{z}_a$ and $B_{K,\mathrm{num}}^{H_b}=\mathbf{z}_b$, let $d_{\mathrm{num}}$ denote the number of numeric fields and let $\kappa_i>0$ denote the normalization range of field $i$. Define
\begin{equation}
D_{\mathrm{num}}(H_a,H_b;K)=\frac{1}{d_{\mathrm{num}}}\sum_{i=1}^{d_{\mathrm{num}}}\frac{|z_{a,i}-z_{b,i}|}{\kappa_i}.
\label{eq:dnum}
\end{equation}

Finally, for the recommended-action components $B_{K,\mathrm{act}}^{H_a}=a_a$ and $B_{K,\mathrm{act}}^{H_b}=a_b$, define
\begin{equation}
D_{\mathrm{act}}(H_a,H_b;K)=\mathbf{1}\{\eta_8(a_a)\neq\eta_8(a_b)\},
\label{eq:dact}
\end{equation}
where $\eta_8$ applies the shared action normalization and returns the first eight tokens of the normalized action representation.
\end{definition}

Equations~\eqref{eq:dcat}--\eqref{eq:dact} instantiate, in order, the five entries of the component vector in Eq.~\eqref{eq:component-divergence}. Each component is bounded in $[0,1]$, symmetric in $H_a$ and $H_b$, and equal to zero when the corresponding components of $B_K^{H_a}$ and $B_K^{H_b}$ agree.

\begin{definition}[Belief Divergence]
\label{def:dbel}
Using the component vector in Eq.~\eqref{eq:component-divergence} and the non-negative weight vector $\mathbf{w}$ specified in Eq.~\eqref{eq:main-dbelief}, the final-step belief divergence is exactly
\begin{equation}
\Dbel(H_a,H_b;K)=\mathbf{w}^{\top}\mathbf{D}(H_a,H_b;K),\qquad \mathbf{w}\succeq\mathbf{0},\qquad \mathbf{1}^{\top}\mathbf{w}=1.
\label{eq:appendix-dbelief}
\end{equation}
\end{definition}
The same $\mathbf{w}$ is fixed across tasks, horizons, and matched harness comparisons.

\begin{proposition}[Basic Properties of $\Dbel$]
\label{prop:pseudometric}
For the matched rollouts of Definition~\ref{def:rollout} and the component distances in Definition~\ref{def:components},
\begin{equation}
0\leq\Dbel(H_a,H_b;K)\leq1.
\label{eq:dbel-bounded}
\end{equation}
Moreover,
\begin{equation}
\Dbel(H,H;K)=0,
\label{eq:dbel-reflexive}
\end{equation}
and
\begin{equation}
\Dbel(H_a,H_b;K)=\Dbel(H_b,H_a;K).
\label{eq:dbel-symmetric}
\end{equation}
\end{proposition}

The proposition follows from the boundedness and symmetry of the five component distances together with the non-negative normalized weighting in Eq.~\eqref{eq:appendix-dbelief}. Distinct harnesses may nevertheless induce identical final beliefs and therefore have zero divergence, so $\Dbel$ is used as a bounded symmetric divergence rather than a metric on the harness space. The proof is given in Appendix~\ref{app:proof-pseudometric}.

\subsection{Arrival and Growth Divergence}
\label{app:arrival-growth}

We now expand the arrival and growth readouts introduced in Eqs.~\eqref{eq:main-darr} and~\eqref{eq:main-dgro}. The construction uses the final beliefs from Definition~\ref{def:rollout}, the five component distances in Definition~\ref{def:components}, and the component vector and weight vector in Eqs.~\eqref{eq:component-divergence} and~\eqref{eq:main-dbelief}.

\begin{definition}[Arrival and Growth Selection]
\label{def:arrival-growth-selection}
Under the component ordering in Eq.~\eqref{eq:component-divergence}, define the selection vectors
\begin{equation}
\mathbf{s}_{\mathrm{arr}}=(0,0,1,0,1)^{\top},\qquad \mathbf{s}_{\mathrm{gro}}=(1,1,0,1,0)^{\top}.
\label{eq:arrival-growth-selectors}
\end{equation}
Thus, $\mathbf{s}_{\mathrm{arr}}$ selects the constraint and action components, while $\mathbf{s}_{\mathrm{gro}}$ selects the categorical, failure, and numeric components. The selectors satisfy
\begin{equation}
\mathbf{s}_{\mathrm{arr}}+\mathbf{s}_{\mathrm{gro}}=\mathbf{1}.
\label{eq:selector-partition}
\end{equation}
For the weight vector $\mathbf{w}$ in Eq.~\eqref{eq:main-dbelief}, define the normalization coefficients as in Eq.~\eqref{eq:arrival-growth-normalizers},
\begin{equation}
\alpha_{\mathrm{arr}}=\frac{1}{\mathbf{w}^{\top}\mathbf{s}_{\mathrm{arr}}},\qquad \alpha_{\mathrm{gro}}=\frac{1}{\mathbf{w}^{\top}\mathbf{s}_{\mathrm{gro}}},
\label{eq:appendix-arrival-growth-normalizers}
\end{equation}
assuming $\mathbf{w}^{\top}\mathbf{s}_{\mathrm{arr}}>0$ and $\mathbf{w}^{\top}\mathbf{s}_{\mathrm{gro}}>0$.
\end{definition}

\begin{definition}[Arrival and Growth Readouts]
\label{def:argro}
For matched harnesses $H_a,H_b$ and horizon $K$, the normalized arrival divergence is
\begin{equation}
\Darr(H_a,H_b;K)=\alpha_{\mathrm{arr}}(\mathbf{w}\odot\mathbf{s}_{\mathrm{arr}})^{\top}\mathbf{D}(H_a,H_b;K),
\label{eq:appendix-darr}
\end{equation}
and the normalized growth divergence is
\begin{equation}
\Dgro(H_a,H_b;K)=\alpha_{\mathrm{gro}}(\mathbf{w}\odot\mathbf{s}_{\mathrm{gro}})^{\top}\mathbf{D}(H_a,H_b;K).
\label{eq:appendix-dgro}
\end{equation}
Here, $\odot$ denotes element-wise multiplication. By Definition~\ref{def:components}, $\Darr$ depends only on $D_{\mathrm{set}}$ and $D_{\mathrm{act}}$, whereas $\Dgro$ depends only on $D_{\mathrm{cat}}$, $D_{\mathrm{fail}}$, and $D_{\mathrm{num}}$.
\end{definition}

Because the selectors partition the five components as in Eq.~\eqref{eq:selector-partition},
\begin{equation}
\mathbf{w}=(\mathbf{w}\odot\mathbf{s}_{\mathrm{arr}})+(\mathbf{w}\odot\mathbf{s}_{\mathrm{gro}}).
\label{eq:weight-partition}
\end{equation}
Combining Eq.~\eqref{eq:weight-partition} with the belief divergence in Eq.~\eqref{eq:main-dbelief} gives
\begin{equation}
\Dbel(H_a,H_b;K)=\alpha_{\mathrm{arr}}^{-1}\Darr(H_a,H_b;K)+\alpha_{\mathrm{gro}}^{-1}\Dgro(H_a,H_b;K),
\label{eq:appendix-decomposition}
\end{equation}
which is the decomposition stated in Eq.~\eqref{eq:main-decomposition}. Since the component distances in Definition~\ref{def:components} lie in $[0,1]$ and each readout renormalizes the selected weights to sum to one,
\begin{equation}
0\leq\Darr(H_a,H_b;K)\leq1,\qquad 0\leq\Dgro(H_a,H_b;K)\leq1.
\label{eq:arrival-growth-bounds}
\end{equation}

\begin{lemma}[Constraint-Set Floor]
\label{lem:floor}
Suppose the constraint components of the final beliefs $B_K^{H_a}$ and $B_K^{H_b}$ from Definition~\ref{def:rollout} satisfy $I_{ab}=0$ and $U_{ab}>0$, where $I_{ab}$ and $U_{ab}$ are defined in Eqs.~\eqref{eq:set-intersection} and~\eqref{eq:set-union}. Then
\begin{equation}
D_{\mathrm{set}}(H_a,H_b;K)=1.
\label{eq:maximal-set-divergence}
\end{equation}
Consequently,
\begin{equation}
\Dbel(H_a,H_b;K)\geq \mathbf{w}^{\top}\mathbf{e}_{\mathrm{set}},
\label{eq:floor}
\end{equation}
\end{lemma}

where $\mathbf{e}_{\mathrm{set}}=(0,0,1,0,0)^{\top}$ selects the constraint-set component under the ordering in Eq.~\eqref{eq:component-divergence}. The condition $I_{ab}=0$ and $U_{ab}>0$ means that the matched final beliefs contain nonempty but disjoint constraint content. By Eq.~\eqref{eq:dset}, this alone makes the constraint-set distance maximal, independently of the remaining four belief components.

\begin{corollary}[Arrival Floor]
\label{cor:floor}
Under the conditions of Lemma~\ref{lem:floor},
\begin{equation}
\Darr(H_a,H_b;K)\geq\alpha_{\mathrm{arr}}\mathbf{w}^{\top}\mathbf{e}_{\mathrm{set}}.
\label{eq:arrival-floor}
\end{equation}
If, in addition, $D_{\mathrm{act}}(H_a,H_b;K)=1$, then
\begin{equation}
\Darr(H_a,H_b;K)=1.
\label{eq:maximal-arrival-divergence}
\end{equation}
\end{corollary}

Lemma~\ref{lem:floor} and Corollary~\ref{cor:floor} show that disjoint constraint content creates a nonzero floor in both the overall and arrival divergences, while simultaneous maximal constraint and action disagreement makes the normalized arrival divergence maximal. Their proofs are given in Appendices~\ref{app:proof-floor} and~\ref{app:proof-growth-floor}, respectively.

\subsection{Matched Comparison and Identification}
\label{app:identification}

\begin{assumption}[Matched Harness Comparison]
\label{ass:matched-comparison}
For matched harnesses $H_a$ and $H_b$, we hold fixed the task $\mathcal{T}$, environment $\mathcal{E}$, base LLM $M$, initial record $h_0$, belief elicitation procedure, decoding procedure, belief schema, and random seed. Thus, the two rollouts differ only through the harness in Eq.~\eqref{eq:harness}:
\begin{equation}
B_{0:K}^{H_a}\sim P(\,\cdot\mid\mathcal{T},\mathcal{E},M,h_0,H_a),\qquad B_{0:K}^{H_b}\sim P(\,\cdot\mid\mathcal{T},\mathcal{E},M,h_0,H_b).
\label{eq:matched-comparison}
\end{equation}
\end{assumption}

\begin{definition}[Component-Level Isolation]
\label{def:component-isolation}
Let $H_a=(H_a^{(1)},\ldots,H_a^{(5)})$ and $H_b=(H_b^{(1)},\ldots,H_b^{(5)})$ follow the component ordering in Eq.~\eqref{eq:harness}. Harness component $j$ is isolated if
\begin{equation}
H_a^{(j)}\neq H_b^{(j)},\qquad H_a^{(i)}=H_b^{(i)}\ \text{for all}\ i\neq j.
\label{eq:component-isolation}
\end{equation}
\end{definition}

\begin{proposition}[Harness-Level Attribution]
\label{prop:harness-attribution}
Under Assumption~\ref{ass:matched-comparison}, if the matched final beliefs satisfy
\begin{equation}
B_K^{H_a}=B_K^{H_b},
\label{eq:identical-final-beliefs}
\end{equation}
then
\begin{equation}
\mathbf{D}(H_a,H_b;K)=\mathbf{0}.
\label{eq:harness-attribution}
\end{equation}
\end{proposition}
If Eq.~\eqref{eq:component-isolation} additionally holds for component $j$, the matched contrast isolates the effect of component $j$. The \BIWM{} intervention in Section~\ref{sec:method} is excluded from this attribution because it additionally changes instrumentation and cross-view evidence. The proof is given in Appendix~\ref{app:proof-harness-attribution}.

\subsection{Censoring and Failure-Mode Divergence}
\label{app:censoring}

\begin{definition}[Censoring Event]
\label{def:censoring-event}
For $H_{\mathrm{raw}}$ and a compared harness $H$, censoring at step $t$ is defined by
\begin{equation}
G_{H_{\mathrm{raw}}}(A_t^{H_{\mathrm{raw}}})\neq\varnothing,\qquad G_H(A_t^H)=\varnothing.
\label{eq:censoring-event}
\end{equation}
\end{definition}

Equation~\eqref{eq:censoring-event} identifies the step at which the two harnesses expose different execution evidence: the raw harness admits its recommended action, while the compared harness blocks its recommendation.

\begin{assumption}[Failure-Mode Censoring Sensitivity]
\label{ass:censoring-sensitivity}
For any censoring event satisfying Eq.~\eqref{eq:censoring-event},
\begin{equation}
\begin{aligned}
&P(B_{t+1,\mathrm{fail}}^{H_{\mathrm{raw}}}\in\cdot\mid B_t^{H_{\mathrm{raw}}},h_t,\mathcal{T},M,H_{\mathrm{raw}})\\
&\neq P(B_{t+1,\mathrm{fail}}^{H}\in\cdot\mid B_t^{H},h_t,\mathcal{T},M,H).
\end{aligned}
\label{eq:censoring-sensitivity}
\end{equation}
\end{assumption}

Assumption~\ref{ass:censoring-sensitivity} requires censoring to affect the conditional distribution of the next failure-mode belief; it does not require the realized failure labels to differ on every rollout.

\begin{lemma}[Censoring-Induced Failure-Mode Divergence]
\label{lem:censor}
Under Assumption~\ref{ass:censoring-sensitivity},
\begin{equation}
\Pr\left(B_{t+1,\mathrm{fail}}^{H_{\mathrm{raw}}}\neq B_{t+1,\mathrm{fail}}^{H}\right)>0,
\label{eq:censor-positive-disagreement}
\end{equation}
and therefore
\begin{equation}
\begin{aligned}
\mathbb{E}\left[D_{\mathrm{fail}}(H_{\mathrm{raw}},H;t+1)\right]
&=\mathbb{E}\left[\mathbf{1}\left\{B_{t+1,\mathrm{fail}}^{H_{\mathrm{raw}}}\neq B_{t+1,\mathrm{fail}}^{H}\right\}\right]\\
&=\Pr\left(B_{t+1,\mathrm{fail}}^{H_{\mathrm{raw}}}\neq B_{t+1,\mathrm{fail}}^{H}\right)\\
&>0.
\end{aligned}
\label{eq:censor-positive-failure-distance}
\end{equation}
\end{lemma}

Lemma~\ref{lem:censor} converts distributional sensitivity into measurable belief divergence: censoring induces a strictly positive expected failure-mode distance at the following step. The proof is given in Appendix~\ref{app:proof-censor}.

\subsection{Growth under Censoring}
\label{app:censoring-growth}

\begin{assumption}[Non-Negative Growth Increment]
\label{ass:growth-increment}
For every $t\geq1$,
\begin{equation}
\mathbb{E}\left[(\mathbf{w}\odot\mathbf{s}_{\mathrm{gro}})^{\top}\mathbf{D}(H_{\mathrm{raw}},H;t+1)-(\mathbf{w}\odot\mathbf{s}_{\mathrm{gro}})^{\top}\mathbf{D}(H_{\mathrm{raw}},H;t)\right]\geq0.
\label{eq:growth-increment}
\end{equation}
\end{assumption}

\begin{theorem}[Censoring-Induced Growth Monotonicity]
\label{thm:monotone}
Under Assumptions~\ref{ass:matched-comparison}, \ref{ass:censoring-sensitivity}, and~\ref{ass:growth-increment}, if Eq.~\eqref{eq:censoring-event} holds for at least one step, then
\begin{equation}
\mathbb{E}\left[\Dgro(H_{\mathrm{raw}},H;K+1)\right]\geq\mathbb{E}\left[\Dgro(H_{\mathrm{raw}},H;K)\right],\quad K\geq1.
\label{eq:monotone}
\end{equation}
\end{theorem}

The result follows from the normalized growth definition in Eq.~\eqref{eq:main-dgro}. The proof is given in Appendix~\ref{app:proof-monotone}.

\section{\BIWM{} Implementation Details}
\label{app:method-details}

This appendix expands the implementation of the three stages in Fig.~\ref{fig:pipeline}. The harness, belief rollout, and divergence quantities are those defined in Eq.~\eqref{eq:harness}, Definition~\ref{def:rollout}, and Eqs.~\eqref{eq:component-divergence}--\eqref{eq:main-decomposition}; Algorithm~\ref{alg:rollout} gives the recursive forecasting procedure.

\subsection{Harness-Specific Record Transformations}
\label{app:harness-realizations}

The six harnesses in Fig.~\ref{fig:pipeline}(a) differ in how they transform the interaction record before each belief update. For a history $h$, we write the harness-visible record compactly as
\begin{equation}
h^{H}=\bigl(O_H(h),A_H,G_H,R_H(h),L_H(h)\bigr).
\label{eq:harness-visible-record}
\end{equation}
The raw harness preserves the original record. The structured harness replaces raw outputs with parsed tracebacks and explicit intermediate states; the risk-gated harness records blocked actions without unrealized outcomes; the repair-heavy harness exposes repaired transitions while compressing preceding failures; the verification-selective harness retains verification evidence only for selected states; and the cost-aware harness omits evidence whose acquisition exceeds the remaining budget.

For component-level contrasts, the active harness pair satisfies Definition~\ref{def:component-isolation}, while all non-harness quantities remain fixed by Assumption~\ref{ass:matched-comparison}. Thus, Eq.~\eqref{eq:harness-visible-record} changes only through the intended harness mechanism.

\subsection{Step-Level Forecast Records}
\label{app:rollout-records}

Algorithm~\ref{alg:rollout} produces one record per forecast step. The stored record contains
\begin{equation}
r_t^H=\bigl(h_t^H,A_t^H,g_t^H,r_t^H,v_t^H,s_t^H,B_t^H\bigr),
\label{eq:step-record}
\end{equation}
where the entries denote the harness-visible record, recommended action, gating metadata, repair metadata, verification evidence, shadow evidence, and structured belief, respectively.

The belief itself follows Definition~\ref{def:rollout}; Eq.~\eqref{eq:belief-transition} governs the recursive update and Eq.~\eqref{eq:belief-action} supplies the recommendation. When a branch is censored, the step record retains the blocked event according to Eq.~\eqref{eq:blocked-rollout}. Each horizon is generated independently from the same matched initial record and seed.

\subsection{\BIWM{} Evidence Preparation}
\label{app:biwm-exposure}

Figure~\ref{fig:pipeline}(b) applies five evidence-preparation operations to the step record in Eq.~\eqref{eq:step-record}. We write their composition as
\begin{equation}
r_{t,\mathrm{prep}}^H=\mathsf{S}\!\circ\mathsf{V}\!\circ\mathsf{R}\!\circ\mathsf{G}\!\circ\mathsf{C}(r_t^H),
\label{eq:biwm-preparation}
\end{equation}
where $\mathsf{C}$ canonicalizes representation, $\mathsf{G}$ exposes blocked branches, $\mathsf{R}$ unrolls repair trajectories, $\mathsf{V}$ annotates verification status, and $\mathsf{S}$ attaches shadow evidence when available.

Canonicalization follows Assumption~\ref{ass:canonical}. Blocked-action logging preserves the censoring semantics in Eq.~\eqref{eq:blocked-rollout}; repair unrolling restores failure--repair--recovery sequences; verification annotation distinguishes checked from unchecked states; and shadow execution records sandboxed counterfactual outcomes without replacing the live record. These operations expose evidence that may be absent or compressed in the original harness record, so they may increase the disagreement measured by Eq.~\eqref{eq:main-dbelief}.

\subsection{\BIWM{} Cross-View Alignment}
\label{app:biwm-alignment}

Figure~\ref{fig:pipeline}(c) aligns the prepared report views associated with the same task, seed, horizon, and forecast step. For categorical fields,
\begin{equation}
B_{t,\mathrm{cat}}^{\mathrm{align}}=\operatorname{mode}_{H} B_{t,\mathrm{cat}}^{H},
\label{eq:biwm-cat-align}
\end{equation}
while numeric fields are aligned by
\begin{equation}
B_{t,\mathrm{num}}^{\mathrm{align}}=\frac{1}{|\mathcal{H}_{\mathrm{match}}|}\sum_{H\in\mathcal{H}_{\mathrm{match}}}B_{t,\mathrm{num}}^{H}.
\label{eq:biwm-num-align}
\end{equation}
Cross-view disagreement is retained as uncertainty rather than discarded. Alignment therefore operates on the same belief components used in Eq.~\eqref{eq:component-divergence} without redefining $\Dbel$, $\Darr$, or $\Dgro$.

\subsection{Evaluation Role and Reproducibility}
\label{app:evaluation-role}

The study evaluates harness-conditioned belief variation rather than policy leaderboard performance. ReAct~\citep{yao2023react}, Reflexion~\citep{shinn2023reflexion}, Toolformer~\citep{schick2023toolformer}, SWE-agent~\citep{yang2024sweagent}, OpenHands~\citep{wang2024openhands}, API-Bank~\citep{li2023apibank}, ToolLLM~\citep{qin2023toolllm}, Gorilla/APIBench~\citep{patil2024gorilla}, and BFCL~\citep{patil2025bfcl} provide relevant reasoning, tool-use, software-engineering, and function-calling references. Public benchmarks are used as stress tests for the measurement protocol.

All metrics are recomputed from the stored step records in Eq.~\eqref{eq:step-record} using Definition~\ref{def:components} and Eqs.~\eqref{eq:main-dbelief}--\eqref{eq:main-decomposition}. This keeps the implementation traceable from harness-visible evidence to the reported arrival, growth, and overall divergence.

\section{Proofs}\label{app:proofs}

\subsection{Proof of Proposition~\ref{prop:pseudometric}}
\label{app:proof-pseudometric}

\begin{proof}
By Definition~\ref{def:components}, each entry of the component vector in Eq.~\eqref{eq:component-divergence} satisfies
\begin{equation}
0\leq D_i(H_a,H_b;K)\leq1,\qquad i\in\{\mathrm{cat},\mathrm{fail},\mathrm{set},\mathrm{num},\mathrm{act}\}.
\label{eq:proof-component-bounds}
\end{equation}
Using $\mathbf{w}\succeq\mathbf{0}$ and $\mathbf{1}^{\top}\mathbf{w}=1$ from Eq.~\eqref{eq:main-dbelief},
\begin{equation}
\begin{aligned}
0 &\leq \mathbf{w}^{\top}\mathbf{D}(H_a,H_b;K)\\
&=\Dbel(H_a,H_b;K)\\
&\leq \mathbf{w}^{\top}\mathbf{1}\\
&=1,
\end{aligned}
\label{eq:proof-dbel-bounded}
\end{equation}
which proves Eq.~\eqref{eq:dbel-bounded}.

For identical harness arguments, Definition~\ref{def:components} gives
\begin{equation}
\begin{aligned}
\mathbf{D}(H,H;K)&=\mathbf{0},\\
\Dbel(H,H;K)&=\mathbf{w}^{\top}\mathbf{D}(H,H;K)\\
&=\mathbf{w}^{\top}\mathbf{0}\\
&=0,
\end{aligned}
\label{eq:proof-dbel-reflexive}
\end{equation}
which proves Eq.~\eqref{eq:dbel-reflexive}.

Each component distance in Definition~\ref{def:components} is symmetric, and hence
\begin{equation}
\begin{aligned}
\mathbf{D}(H_a,H_b;K)&=\mathbf{D}(H_b,H_a;K),\\
\Dbel(H_a,H_b;K)&=\mathbf{w}^{\top}\mathbf{D}(H_a,H_b;K)\\
&=\mathbf{w}^{\top}\mathbf{D}(H_b,H_a;K)\\
&=\Dbel(H_b,H_a;K),
\end{aligned}
\label{eq:proof-dbel-symmetric}
\end{equation}
which proves Eq.~\eqref{eq:dbel-symmetric}.

Finally, for distinct harnesses $H_a\neq H_b$ satisfying $B_K^{H_a}=B_K^{H_b}$,
\begin{equation}
\begin{aligned}
B_K^{H_a}=B_K^{H_b}&\quad\text{and}\quad H_a\neq H_b,\\
\mathbf{D}(H_a,H_b;K)&=\mathbf{0},\\
\Dbel(H_a,H_b;K)&=\mathbf{w}^{\top}\mathbf{0}\\
&=0.
\end{aligned}
\label{eq:proof-nonidentity}
\end{equation}
Thus, distinct harnesses can have zero belief divergence, so identity of indiscernibles does not hold on the harness space.
\end{proof}

\subsection{Proof of Lemma~\ref{lem:floor}}
\label{app:proof-floor}

\begin{proof}
By the assumptions $I_{ab}=0$ and $U_{ab}>0$, together with Eq.~\eqref{eq:dset},
\begin{equation}
\begin{aligned}
D_{\mathrm{set}}(H_a,H_b;K)
&=1-\frac{I_{ab}}{U_{ab}}\\
&=1-\frac{0}{U_{ab}}\\
&=1,
\end{aligned}
\label{eq:proof-maximal-set-divergence}
\end{equation}
which proves Eq.~\eqref{eq:maximal-set-divergence}.

Using Eq.~\eqref{eq:main-dbelief}, the non-negativity of every entry of $\mathbf{w}$ and $\mathbf{D}(H_a,H_b;K)$, and the selector $\mathbf{e}_{\mathrm{set}}$ from Lemma~\ref{lem:floor},
\begin{equation}
\begin{aligned}
\Dbel(H_a,H_b;K)
&=\mathbf{w}^{\top}\mathbf{D}(H_a,H_b;K)\\
&\geq (\mathbf{w}\odot\mathbf{e}_{\mathrm{set}})^{\top}\mathbf{D}(H_a,H_b;K)\\
&=(\mathbf{w}^{\top}\mathbf{e}_{\mathrm{set}})\,D_{\mathrm{set}}(H_a,H_b;K)\\
&=\mathbf{w}^{\top}\mathbf{e}_{\mathrm{set}},
\end{aligned}
\label{eq:proof-floor}
\end{equation}
which proves Eq.~\eqref{eq:floor}.
\end{proof}

\subsection{Proof of Corollary~\ref{cor:floor}}
\label{app:proof-growth-floor}

\begin{proof}
By Lemma~\ref{lem:floor},
\begin{equation}
D_{\mathrm{set}}(H_a,H_b;K)=1.
\label{eq:proof-arrival-set-one}
\end{equation}
Using Eq.~\eqref{eq:appendix-darr},
\begin{equation}
\begin{aligned}
\Darr(H_a,H_b;K)
&=\alpha_{\mathrm{arr}}(\mathbf{w}\odot\mathbf{s}_{\mathrm{arr}})^{\top}\mathbf{D}(H_a,H_b;K)\\
&\geq\alpha_{\mathrm{arr}}(\mathbf{w}\odot\mathbf{e}_{\mathrm{set}})^{\top}\mathbf{D}(H_a,H_b;K)\\
&=\alpha_{\mathrm{arr}}(\mathbf{w}^{\top}\mathbf{e}_{\mathrm{set}})D_{\mathrm{set}}(H_a,H_b;K)\\
&=\alpha_{\mathrm{arr}}\mathbf{w}^{\top}\mathbf{e}_{\mathrm{set}},
\end{aligned}
\label{eq:proof-arrival-floor}
\end{equation}
which proves Eq.~\eqref{eq:arrival-floor}.

If, in addition, $D_{\mathrm{act}}(H_a,H_b;K)=1$, then the two components selected by $\mathbf{s}_{\mathrm{arr}}$ both equal one, and hence
\begin{equation}
\begin{aligned}
\Darr(H_a,H_b;K)
&=\alpha_{\mathrm{arr}}(\mathbf{w}\odot\mathbf{s}_{\mathrm{arr}})^{\top}\mathbf{1}\\
&=\alpha_{\mathrm{arr}}\mathbf{w}^{\top}\mathbf{s}_{\mathrm{arr}}\\
&=\frac{\mathbf{w}^{\top}\mathbf{s}_{\mathrm{arr}}}{\mathbf{w}^{\top}\mathbf{s}_{\mathrm{arr}}}\\
&=1,
\end{aligned}
\label{eq:proof-maximal-arrival-divergence}
\end{equation}
which proves Eq.~\eqref{eq:maximal-arrival-divergence}.
\end{proof}

\subsection{Proof of Lemma~\ref{lem:censor}}
\label{app:proof-censor}

\begin{proof}
Suppose that a coupling of the two conditional distributions in Eq.~\eqref{eq:censoring-sensitivity} satisfies
\begin{equation}
\Pr\left(B_{t+1,\mathrm{fail}}^{H_{\mathrm{raw}}}\neq B_{t+1,\mathrm{fail}}^{H}\right)=0.
\label{eq:proof-censor-zero}
\end{equation}
Then
\begin{equation}
\begin{aligned}
\Pr\left(B_{t+1,\mathrm{fail}}^{H_{\mathrm{raw}}}=B_{t+1,\mathrm{fail}}^{H}\right)&=1,\\
P(B_{t+1,\mathrm{fail}}^{H_{\mathrm{raw}}}\in\cdot\mid B_t^{H_{\mathrm{raw}}},h_t,\mathcal{T},M,H_{\mathrm{raw}})&=P(B_{t+1,\mathrm{fail}}^{H}\in\cdot\mid B_t^{H},h_t,\mathcal{T},M,H),
\end{aligned}
\label{eq:proof-censor-equal-marginals}
\end{equation}
which contradicts Eq.~\eqref{eq:censoring-sensitivity}. Hence
\begin{equation}
\Pr\left(B_{t+1,\mathrm{fail}}^{H_{\mathrm{raw}}}\neq B_{t+1,\mathrm{fail}}^{H}\right)>0.
\label{eq:proof-censor-positive}
\end{equation}
By Eq.~\eqref{eq:dfail},
\begin{equation}
\begin{aligned}
\mathbb{E}\left[D_{\mathrm{fail}}(H_{\mathrm{raw}},H;t+1)\right]
&=\mathbb{E}\left[\mathbf{1}\left\{B_{t+1,\mathrm{fail}}^{H_{\mathrm{raw}}}\neq B_{t+1,\mathrm{fail}}^{H}\right\}\right]\\
&=\Pr\left(B_{t+1,\mathrm{fail}}^{H_{\mathrm{raw}}}\neq B_{t+1,\mathrm{fail}}^{H}\right)\\
&>0,
\end{aligned}
\label{eq:proof-censor-failure-distance}
\end{equation}
which proves Eqs.~\eqref{eq:censor-positive-disagreement} and~\eqref{eq:censor-positive-failure-distance}.
\end{proof}

\subsection{Proof of Proposition~\ref{prop:harness-attribution}}
\label{app:proof-harness-attribution}

\begin{proof}
By Eq.~\eqref{eq:identical-final-beliefs},
\begin{equation}
B_K^{H_a}=B_K^{H_b}.
\label{eq:proof-belief-equality}
\end{equation}
Since the belief representation in Definition~\ref{def:beliefspace} is componentwise, for every
$c\in\{\mathrm{cat},\mathrm{fail},\mathrm{set},\mathrm{num},\mathrm{act}\}$,
\begin{equation}
B_{K,c}^{H_a}=B_{K,c}^{H_b}.
\label{eq:proof-belief-component-equality}
\end{equation}
Therefore, by Definition~\ref{def:components},
\begin{equation}
\begin{aligned}
D_c(H_a,H_b;K) &= 0,
&& c\in\{\mathrm{cat},\mathrm{fail},\mathrm{set},\mathrm{num},\mathrm{act}\},\\
\mathbf{D}(H_a,H_b;K)
&=
\begin{pmatrix}
D_{\mathrm{cat}}\\
D_{\mathrm{fail}}\\
D_{\mathrm{set}}\\
D_{\mathrm{num}}\\
D_{\mathrm{act}}
\end{pmatrix}
=
\begin{pmatrix}
0\\0\\0\\0\\0
\end{pmatrix}
=\mathbf{0}.
\end{aligned}
\label{eq:proof-harness-attribution}
\end{equation}

For component-level isolation, Eq.~\eqref{eq:component-isolation} gives
\begin{equation}
\begin{aligned}
H_a^{(j)} &\neq H_b^{(j)},\\
H_a^{(i)} &= H_b^{(i)}, && i\neq j.
\end{aligned}
\label{eq:proof-isolated-harness}
\end{equation}
Assumption~\ref{ass:matched-comparison} additionally fixes all non-harness quantities:
\begin{equation}
\begin{aligned}
\mathcal{T}_a &= \mathcal{T}_b, &
\mathcal{E}_a &= \mathcal{E}_b,\\
M_a &= M_b, &
h_{0,a} &= h_{0,b}.
\end{aligned}
\label{eq:proof-isolated-context}
\end{equation}
Thus, the two matched comparisons differ only in harness component $H^{(j)}$, which establishes component-level isolation.
\end{proof}

\subsection{Proof of Theorem~\ref{thm:monotone}}
\label{app:proof-monotone}

\begin{proof}
By Eq.~\eqref{eq:main-dgro}, the growth divergences at horizons $K+1$ and $K$ satisfy
\begin{equation}
\begin{aligned}
\Dgro(H_{\mathrm{raw}},H;K+1)
&=\alpha_{\mathrm{gro}}(\mathbf{w}\odot\mathbf{s}_{\mathrm{gro}})^{\top}\mathbf{D}(H_{\mathrm{raw}},H;K+1),\\
\Dgro(H_{\mathrm{raw}},H;K)
&=\alpha_{\mathrm{gro}}(\mathbf{w}\odot\mathbf{s}_{\mathrm{gro}})^{\top}\mathbf{D}(H_{\mathrm{raw}},H;K).
\end{aligned}
\label{eq:proof-growth-two-horizons}
\end{equation}
Their difference is therefore
\begin{equation}
\begin{aligned}
&\Dgro(H_{\mathrm{raw}},H;K+1)-\Dgro(H_{\mathrm{raw}},H;K)\\
&=\alpha_{\mathrm{gro}}(\mathbf{w}\odot\mathbf{s}_{\mathrm{gro}})^{\top}\mathbf{D}(H_{\mathrm{raw}},H;K+1)-\alpha_{\mathrm{gro}}(\mathbf{w}\odot\mathbf{s}_{\mathrm{gro}})^{\top}\mathbf{D}(H_{\mathrm{raw}},H;K)\\
&=\alpha_{\mathrm{gro}}\left[(\mathbf{w}\odot\mathbf{s}_{\mathrm{gro}})^{\top}\mathbf{D}(H_{\mathrm{raw}},H;K+1)-(\mathbf{w}\odot\mathbf{s}_{\mathrm{gro}})^{\top}\mathbf{D}(H_{\mathrm{raw}},H;K)\right].
\end{aligned}
\label{eq:proof-growth-difference}
\end{equation}
Taking expectations and using linearity gives
\begin{equation}
\begin{aligned}
&\mathbb{E}\left[\Dgro(H_{\mathrm{raw}},H;K+1)-\Dgro(H_{\mathrm{raw}},H;K)\right]\\
&=\mathbb{E}\left[\alpha_{\mathrm{gro}}\left((\mathbf{w}\odot\mathbf{s}_{\mathrm{gro}})^{\top}\mathbf{D}(H_{\mathrm{raw}},H;K+1)-(\mathbf{w}\odot\mathbf{s}_{\mathrm{gro}})^{\top}\mathbf{D}(H_{\mathrm{raw}},H;K)\right)\right]\\
&=\alpha_{\mathrm{gro}}\mathbb{E}\left[(\mathbf{w}\odot\mathbf{s}_{\mathrm{gro}})^{\top}\mathbf{D}(H_{\mathrm{raw}},H;K+1)-(\mathbf{w}\odot\mathbf{s}_{\mathrm{gro}})^{\top}\mathbf{D}(H_{\mathrm{raw}},H;K)\right]\\
&\geq0,
\end{aligned}
\label{eq:proof-expected-growth-difference}
\end{equation}
where the last inequality follows from Assumption~\ref{ass:growth-increment} and
\begin{equation}
\begin{aligned}
\alpha_{\mathrm{gro}}
&=\frac{1}{\mathbf{w}^{\top}\mathbf{s}_{\mathrm{gro}}}\\
&>0.
\end{aligned}
\label{eq:proof-positive-growth-normalizer}
\end{equation}
Finally,
\begin{equation}
\begin{aligned}
0
&\leq\mathbb{E}\left[\Dgro(H_{\mathrm{raw}},H;K+1)-\Dgro(H_{\mathrm{raw}},H;K)\right]\\
&=\mathbb{E}\left[\Dgro(H_{\mathrm{raw}},H;K+1)\right]-\mathbb{E}\left[\Dgro(H_{\mathrm{raw}},H;K)\right],
\end{aligned}
\label{eq:proof-growth-monotone-difference}
\end{equation}
and therefore
\begin{equation}
\begin{aligned}
\mathbb{E}\left[\Dgro(H_{\mathrm{raw}},H;K+1)\right]
&\geq\mathbb{E}\left[\Dgro(H_{\mathrm{raw}},H;K)\right],\qquad K\geq1.
\end{aligned}
\label{eq:proof-monotone}
\end{equation}
This proves Eq.~\eqref{eq:monotone}.
\end{proof}

\end{document}